\pdfoutput=1
\documentclass{article}
\usepackage{tkai}
\usepackage{float}
\usepackage[T1]{fontenc}
\usepackage[utf8]{inputenc}
\usepackage{times}  
\usepackage{helvet} 
\usepackage{courier}  
\usepackage[hyphens]{url}  
\usepackage{graphicx} 
\usepackage{caption} 
\usepackage{caption} 
\usepackage{algorithm}
\usepackage{algorithmic}
\usepackage{enumitem}
\usepackage{booktabs}       
\usepackage{amsfonts}       
\usepackage{nicefrac}      
\usepackage{xcolor}         
\usepackage{amsmath}
\usepackage{amssymb}
\usepackage{color}
\usepackage{graphics}
\usepackage{lipsum}
\usepackage{subcaption}
\newtheorem{theorem}{Theorem}

\usepackage{breqn}
\newenvironment{proof}{\paragraph{Proof:}}{\hfill$\square$}

\usepackage{newfloat}
\usepackage{listings}
\DeclareCaptionStyle{ruled}{labelfont=normalfont,labelsep=colon,strut=off} 
\floatstyle{ruled}
\newfloat{listing}{tb}{lst}{}
\floatname{listing}{Listing}

\title{Stable and Robust Deep Learning By Hyperbolic Tangent Exponential Linear Unit (TeLU)}
\author{%
  Alfredo Fernandez \\
  University of South Florida\\
  \texttt{afernandez7@usf.edu}
  \And
  Ankur Mali \\
  University of South Florida\\
  \texttt{ankurarjunmali@usf.edu}
}

\usepackage{bibentry}

\begin{document}

\maketitle

\begin{abstract}
In this paper, we introduce the Hyperbolic Tangent Exponential Linear Unit (TeLU), a novel neural network activation function, represented as $f(x) = x{\cdot}tanh(e^x)$. TeLU is designed to overcome the limitations of conventional activation functions like ReLU, GELU, and Mish by addressing the vanishing and, to an extent, the exploding gradient problems. Our theoretical analysis and empirical assessments reveal that TeLU outperforms existing activation functions in stability and robustness, effectively adjusting activation outputs' mean towards zero for enhanced training stability and convergence. Extensive evaluations against popular activation functions (ReLU, GELU, SiLU, Mish, Logish, Smish) across advanced architectures, including Resnet-50, demonstrate TeLU's lower variance and superior performance, even under hyperparameter conditions optimized for other functions. In large-scale tests with challenging datasets like CIFAR-10, CIFAR-100, and TinyImageNet, encompassing 860 scenarios, TeLU consistently showcased its effectiveness, positioning itself as a potential new standard for neural network activation functions, boosting stability and performance in diverse deep learning applications.

\end{abstract}

\section{Introduction}
\label{intro}
In the rapidly evolving landscape of neural networks, the choice of activation function plays a pivotal role in model performance and stability. While the Rectified Linear Unit (ReLU) \cite{ReLU, nair2010rectified} has long been the cornerstone of numerous deep learning architectures \cite{silver2017mastering, ResNet, simonyan2014very} due to its simplicity and effectiveness in mitigating the vanishing gradient problem \cite{hochreiter1991untersuchungen, hochreiter2001gradient}, it is not without limitations. Particularly, ReLU suffers from the "dying ReLU" issue \cite{lu2019dying}, where neurons can become inactive and cease to contribute to the learning process, potentially leading to suboptimal models.

Enter the Gaussian Error Linear Unit (GELU) \cite{GELU} and Mish \cite{Mish} activation functions, which have emerged as sophisticated alternatives, addressing some of ReLU's shortcomings. GELU, leveraging the properties of the Gaussian distribution, offers a smooth, non-linear transition in its activation, which can lead to improved learning dynamics \cite{Transformer,BERT,GPT}. Mish, further building on this concept, introduces a self-gating mechanism, enabling a smoother information flow. However, both GELU and Mish, despite their advancements, bring increased computational complexity and lack specific theoretical guarantees, particularly in the context of network stability and convergence.

This is where the Hyperbolic Tangent Exponential Linear Unit (TeLU) marks a significant stride forward. TeLU, not only addresses the aforementioned limitations but also introduces compelling theoretical advantages. Its formulation ensures a balance between linearity and non-linearity, offering the best of both worlds: the simplicity and robustness of ReLU and the smooth, gradient-nurturing properties of GELU and Mish. The unique composition of TeLU, particularly the hyperbolic tangent of the exponential function, provides a natural regulation of the activation's magnitude, effectively sidestepping issues like exploding gradients.

Moreover, TeLU's most notable distinction lies in its theoretical underpinnings. It demonstrates remarkable properties in the context of the Fisher Information Matrix, contributing to a smoother optimization landscape. This characteristic is crucial for deep learning models, as it directly correlates with more stable and efficient training dynamics, leading to enhanced convergence properties. In essence, TeLU paves the way for theoretically sound and empirically robust neural network designs, potentially setting a new standard in the realm of activation functions.

This paper is organized as follows: Section 2 outlines the proposed TeLU activation function and mathematical analysis, Section 3 describes the experimental setup, section 4 presents results and discussion and section 5 contains the final conclusion remarks.

\section{TeLU Formulation and Mathematical analysis}

\begin{figure}
   
    \includegraphics [width=0.75\linewidth]{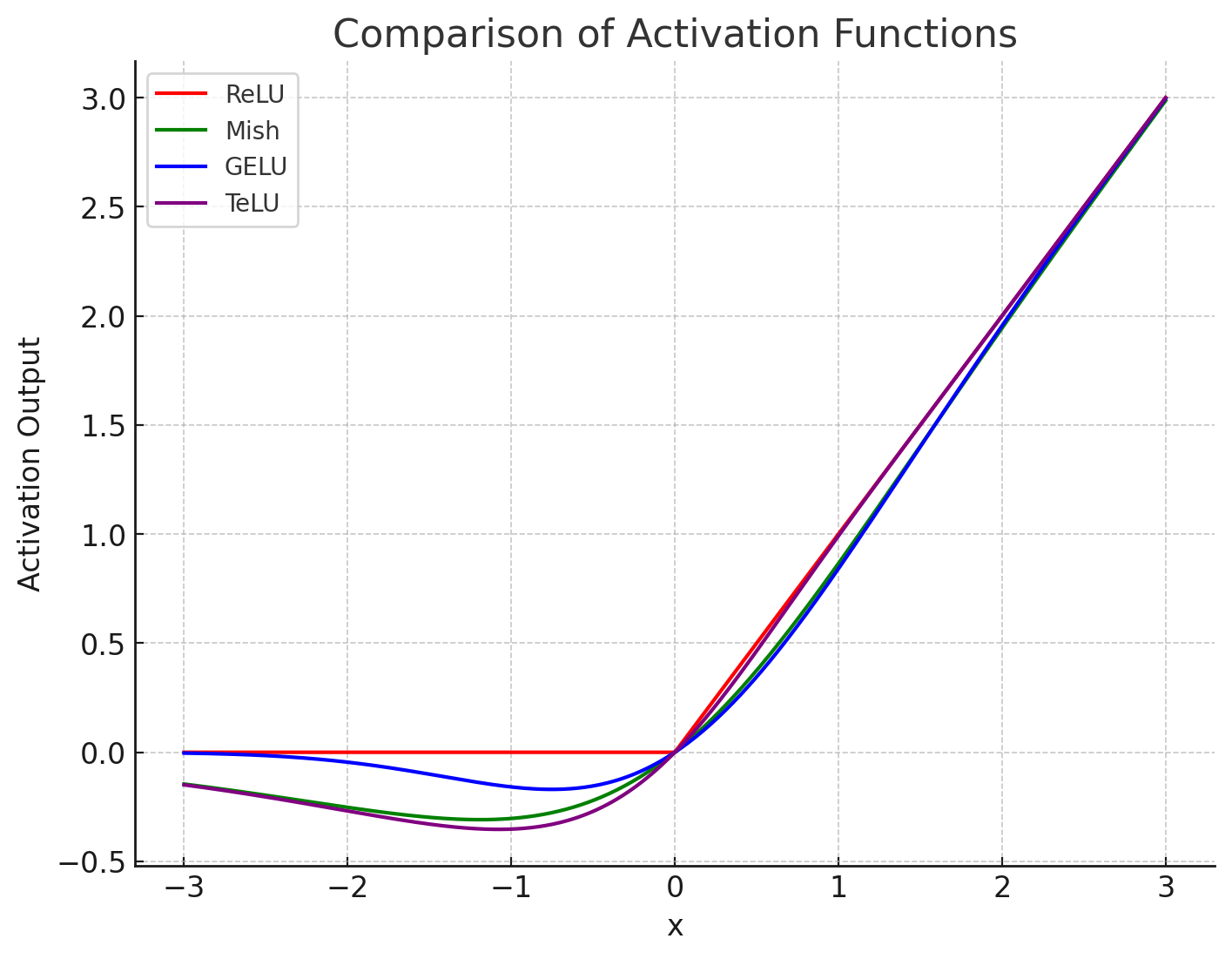}
    \caption{The characteristic of the TeLU activation function along with ReLU, GELU and Mish.}
    \label{fig:telu_comp}
\end{figure}

The Hyperbolic Tangent Exponential Linear Unit (TeLU) activation function represents a notable advancement in neural network design, marrying practical performance with theoretical robustness. Mathematically TeLU is represented as follows:

\begin{align}
    f(x) = x{\cdot}tanh(e^x)
\end{align}

TeLU elegantly integrates the linear characteristics of traditional activation functions with the non-linear benefits of exponential and hyperbolic tangent functions. This fusion ensures that TeLU maintains a balance between facilitating efficient learning and preventing gradient-related issues (credit assignment) commonly encountered in deep neural networks. At the heart of TeLU's design is the hyperbolic tangent of the exponential function, which intuitively moderates the activation's output, ensuring it remains within a manageable range. This characteristic is crucial in mitigating the risk of exploding gradients, a common pitfall in deep network training. Moreover, unlike some of its predecessors, TeLU offers a smooth transition across the origin, which enhances the gradient flow through the network. This smoothness is particularly beneficial in deep learning models, as it contributes to more stable and consistent learning dynamics. This can be visualized in Figure \ref{fig:telu_comp}, which shows the continuity of the TeLU and also that it saturates at a lower rate compared to other SoTA functions.

Furthermore, TeLU's formulation brings theoretical benefits, particularly in the Fisher Information Matrix (FIM) context. This aspect of TeLU underpins a smoother optimization landscape, a property that directly correlates with enhanced training stability and convergence. It is evident from figure \ref{fig:telu_div}, where the second derivative of Mish saturates, whereas GELU and TeLU are much more stable. One important thing to note is that TeLU, for large values, comes closer to GELU, which can also validate its empirical performance.  

\subsection{Mathematical Analysis}
In this section, we mathematically prove several properties of TeLU, including credit assignment issues, stability, robustness, and convergence.

\begin{figure}
    \includegraphics [width=0.8\linewidth]{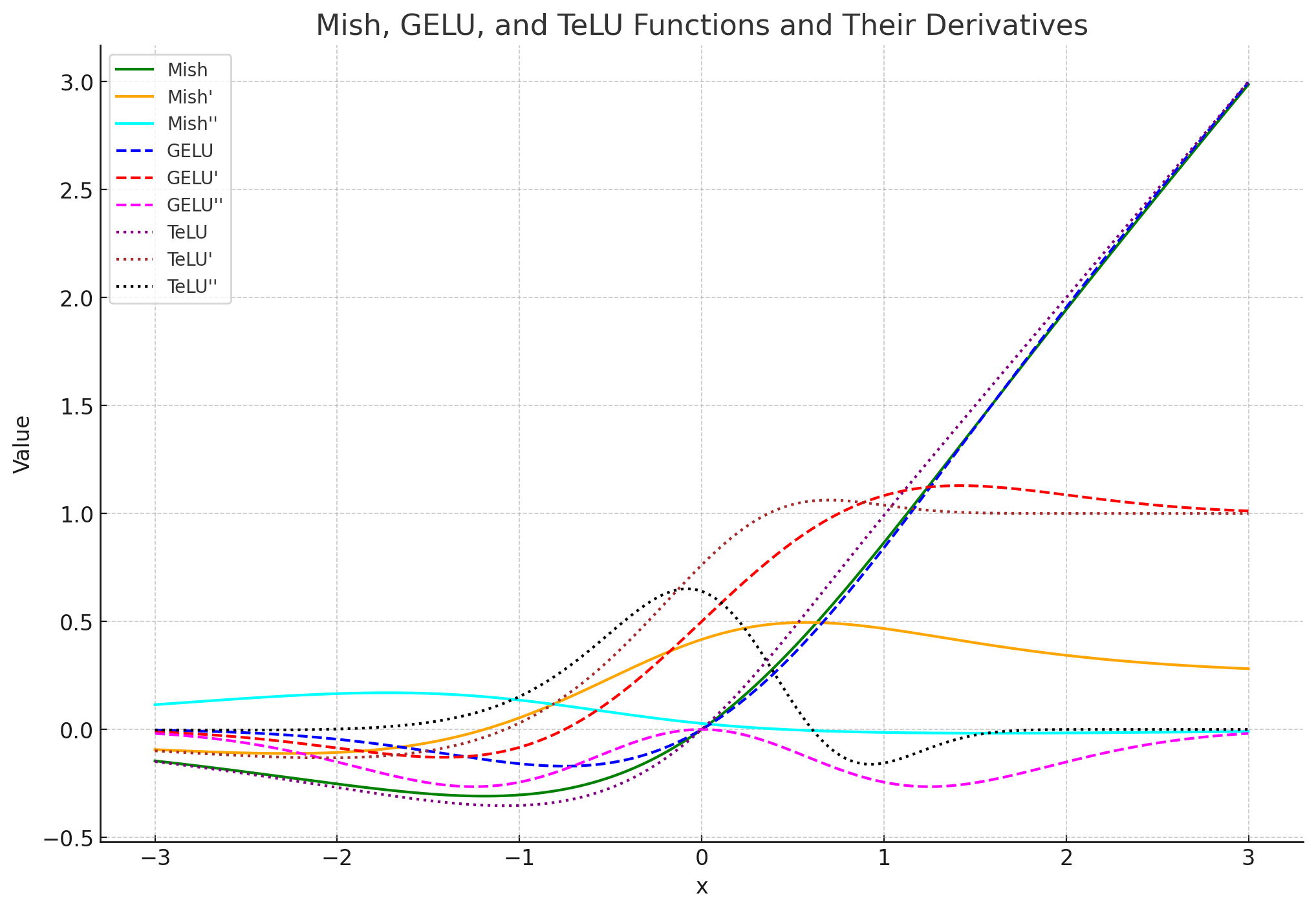}
    \caption{The first and second derivative of proposed TeLU activation compared to derivatives of GELU and Mish }
    \label{fig:telu_div}
\end{figure}

Let $\sigma$ be an activation function given as $y = \sigma(x)$, where x is the input and y is the output. Let $\mathcal{F}(\Theta)$ be the set of parameters using the $\sigma$ non-linearities. Let the function  $\mathcal{f}$ be optimized by the objective function $\mathcal{L}(\Theta)$ using standard backpropagation of error, then we show $\sigma$ applied on any function ${f}$ avoids vanishing gradients issues in the neural network.

\begin{theorem} \label{thm:1}
    
    If  {f}(x) = $x \cdot \tanh(e^x)$, then it avoids gradient vanishing problem since  $f'(x) \neq 0 \text{ for all } x \in \mathbb{R}$.
\end{theorem}
\begin{proof}

The derivative of \( {f}(x)\) with respect to \( x \) is given by:

\[ {f}'(x) = \frac{d}{dx} \left( x \cdot \tanh(e^x) \right). \]

Applying the product rule and the chain rule, we find:

\[ f'(x) = \tanh(e^x) + x \cdot (1 - \tanh^2(e^x)) \cdot e^x. \]

We analyze this derivative of above function in two parts:
\begin{itemize}
    \item \( \tanh(e^x) \) is always non-negative, as for some value of z the \( \tanh(z) \) is bounded between -1 and 1 for all \( z \) and \( e^x \) is always positive.
    \item \( 1 - \tanh^2(e^x) \) is always positive since \( |\tanh(z)| < 1 \) for all \( z \), and \( e^x \) is always positive for all real \( x \)
\end{itemize}

Thus, the second term \( x \cdot (1 - \tanh^2(e^x)) \cdot e^x \) is always non-zero unless \( x = 0 \). However, even at \( x = 0 \), the first term \( \tanh(e^x) \) remains non-zero. Therefore, the entire expression for \( f'(x) \) is non-zero for all \( x \).

Hence, we conclude:

\[ f'(x) \neq 0 \text{ for all } x \in \mathbb{R}. \]

\end{proof}

Next, we show TeLU exhibits the saturating behavior thus under mild assumption, can avoid exploding gradient issues in neural network 
\begin{theorem}
    Let \( f(x) = x \cdot \tanh(e^x) \). Then for \( x > 0 \), \( f(x) \) exhibits controlled growth, and for \( x \leq 0 \), \( f(x) \) shows saturating behavior. The derivative \( f'(x) \) remains finite for all \( x \in \mathbb{R} \), contributing to the mitigation of exploding gradients.

\end{theorem}
\begin{proof}
We now analyze the function \( f(x) \) and its derivative \( f'(x) \) in two regions: for \( x > 0 \) and \( x \leq 0 \).

\textit{Controlled Growth for Positive Values:}

For \( x > 0 \), the exponential function \( e^x \) grows rapidly. However, the hyperbolic tangent function \( \tanh(z) \) is bounded and saturates, where \( \lim_{z \to \infty} \tanh(z) = 1 \). Therefore, for large positive values of \( x \), \( f(x) \) grows linearly, as \( f(x) \approx x \) due to the saturation of \( \tanh(e^x) \) towards 1. This linear growth prevents the function from exhibiting exponential growth with a bound, thus mitigating the risk of exploding gradients.

\textit{Saturating Behavior for Negative Values:}

For \( x \leq 0 \), as \( x \to -\infty \), the term \( e^x \) approaches 0, causing \( \tanh(e^x) \) to also approach 0. Consequently, \( f(x) \) approaches 0, showing a saturating behavior as \( x \) becomes large in the negative direction.

Now Lets consider the derivative \( f'(x) \):

\[ f'(x) = \tanh(e^x) + x \cdot (1 - \tanh^2(e^x)) \cdot e^x, \]

we observe that for positive \( x \), \( f'(x) \) remains finite due to the saturation of \( \tanh(e^x) \) and the controlled growth of \( x \cdot (1 - \tanh^2(e^x)) \cdot e^x \). For negative \( x \), the derivative tends towards 0, reflecting the saturating behavior of \( f(x) \).

Hence, we conclude that \( f(x) = x \cdot \tanh(e^x) \) exhibits controlled growth for positive values and saturating behavior for negative values, which contributes to avoiding exploding gradients for positive values of \( x \) within a bound.
\end{proof}

Next, we show that TeLU has an implicit regularization, thus avoiding overfitting, exhibits stable behavior, zero-mean activation \cite{raiko2012deep} and converges faster. 
\begin{theorem}
    Let \( f(x) = x \cdot \tanh(e^x) \). If \( x \) is a random variable following a symmetric probability distribution about zero, then the expected value (mean) of \( f(x) \) is approximately zero, and f(x) provides efficient gradient flow and implicit regularization.
    \begin{equation*}
    \lim_{a \to \infty} \frac{1}{2a} \int_{-a}^{a} f(x) \, dx = 0
\end{equation*}
\end{theorem}





\begin{proof}

\begin{align*}
1. \quad & \text{Let } f(x) = x \cdot \tanh(e^x). \\
2. \quad & \text{Near Zero: } \lim_{x \to 0} f(x) = \lim_{x \to 0} x \cdot \tanh(e^x) = 0. \\
3. \quad & \text{Away from Zero: } \lim_{x \to \pm \infty} f(x) = \pm \infty. \\
4. \quad & \text{Consider } I(a) = \int_{-a}^{a} f(x) \, dx. \\
5. \quad & \text{For large } a, \text{ the positive and negative values of } f(x) \text{ over } [-a, a] \text{ counterbalance each other.} \\
6. \quad & \text{Therefore, } \lim_{a \to \infty} \frac{1}{2a} I(a) = \lim_{a \to \infty} \frac{1}{2a} \int_{-a}^{a} f(x) \, dx = 0. \\
\end{align*}

Thus the function \( f(x) = x \cdot \tanh(e^x) \) has an asymptotic mean-shifting property towards zero.
\end{proof}

In the appendix \ref{thm:relu}, we show mathematically that the mean of the activation for ReLU doesn't approach zero. Next, we prove the network's stability and explain why TeLU has the lowest variance among competing activation functions.
\begin{theorem}
The function \( f(x) = x \cdot \tanh(e^x) \) exhibits stable behavior for any neural network.
\end{theorem}
\begin{proof}
    \textit{Bounded Output:} The hyperbolic tangent function \( \tanh(z) \) has outputs bounded between -1 and 1. Therefore, for any real number \( x \), the product \( x \cdot \tanh(e^x) \) will not grow unbounded, contributing to stability. Mathematically, this can be expressed as:

   \[ -|x| \leq f(x) \leq |x| \]

2. \textit{Non-zero Gradient:} The derivative of \( f(x) \), given by

   \[ f'(x) = \tanh(e^x) + x \cdot (1 - \tanh^2(e^x)) \cdot e^x \]

   is always non-zero for all real numbers \( x \). This ensures that the gradients do not vanish during backpropagation, which is crucial for stable learning in deep networks.

3. \textit{Controlled Growth for Positive \( x \):} For positive \( x \), the function grows linearly since \( \tanh(e^x) \) approaches 1. This linear growth is more stable than exponential growth, which could lead to exploding gradients.

4. \textit{Saturating Behavior for Negative \( x \):} For negative \( x \), as \( x \) becomes large in the negative direction, \( f(x) \) approaches 0. This saturation helps prevent the function from contributing to exploding gradients during training.

Therefore, due to its bounded output, non-zero gradient, controlled growth for positive values, and saturating behavior for negative values, the function \( f(x) = x \cdot \tanh(e^x) \) is shown to be stable in the context of neural network activations.
\end{proof}

Next, we show TeLU is more robust to small noise and perturbations compared to ReLU, which is an important property to design adversarial-resistant neural network
\begin{theorem}
    The function \( f(x) = x \cdot \tanh(e^x) \) is more robust compared to Relu ($g(x) = max(0,x)$) and robust against small perturbations or noise in the input.
\end{theorem}
\begin{proof}
We analyze the derivative of \( f(x) \) to show robustness to small perturbations. The derivative gives the rate of change of the function with respect to changes in the input. A small derivative magnitude indicates robustness to small changes or noise in the input.
The derivative of g(x) = Relu is represented as follows:
\[
\text{g}'(x) = 
\begin{cases} 
0 & \text{if } x < 0 \\
1 & \text{if } x > 0 \\
\text{undefined} & \text{if } x = 0 
\end{cases}
\]

This derivative shows that for \( x > 0 \), the function is sensitive to changes, as even small positive changes in \( x \) will result in a change in output. The function is insensitive to changes for \( x < 0 \), as the output remains zero. The derivative is undefined at \( x = 0 \), indicating a discontinuity, which can be problematic for stability.

The derivative of \( f(x) = TeLU \) is given by:

\[ f'(x) = \tanh(e^x) + x \cdot (1 - \tanh^2(e^x)) \cdot e^x \]

Consider the behavior of \( f'(x) \) for different ranges of \( x \):

\textit{For large negative \( x \)}: As \( x \) becomes very negative, \( e^x \) approaches 0, making \( \tanh(e^x) \) and its derivative small. Thus, \( f'(x) \) becomes small, indicating that \( f(x) \) is not highly sensitive to small changes in \( x \).

\textit{For small \( x \) around 0}: Here, \( \tanh(e^x) \) is approximately equal to \( e^x \), which is close to 1 for small \( x \). The term \( x \cdot (1 - \tanh^2(e^x)) \cdot e^x \) is also small. Hence, \( f'(x) \) remains moderate, suggesting that \( f(x) \) does not change drastically for small perturbations around 0.

\textit{For large positive \( x \)}: Although \( e^x \) grows, the term \( \tanh(e^x) \) approaches 1, limiting the growth of \( f(x) \). The term \( x \cdot (1 - \tanh^2(e^x)) \cdot e^x \) becomes small as \( x \) increases, due to the saturation of \( \tanh(e^x) \). Thus, \( f'(x) \) remains bounded.

Since \( f'(x) \) does not exhibit large values across the range of \( x \), it indicates that \( f(x) \) does not change disproportionately for small changes in \( x \), thereby demonstrating robustness to small perturbations or noise.
\end{proof}

Next, we show a strong property which shows TeLU is Lipschitz continuous, which is important to uniform continuity of the function
\begin{theorem}
    The function \( f: \mathbb{R} \to \mathbb{R} \), defined by \( f(x) = x \cdot \tanh(e^x) \), is Lipschitz continuous on the real line \( \mathbb{R} \).
\end{theorem}
\begin{proof}
    To demonstrate that \( f \) is Lipschitz continuous, we seek a constant \( L \) such that for all \( x, y \in \mathbb{R} \), the inequality

\[ |f(x) - f(y)| \leq L |x - y| \]

is satisfied. A sufficient condition for this is that the derivative of \( f \), \( f'(x) \), is bounded on \( \mathbb{R} \).

The derivative of \( f \) is given by

\[ f'(x) = \tanh(e^x) + x \cdot \frac{e^x}{\cosh^2(e^x)} \]

We analyze the boundedness of \( f'(x) \) in two parts:

1. The function \( \tanh(e^x) \) is bounded on \( \mathbb{R} \) as \( \tanh \) outputs values in \((-1, 1)\).

2. For the term \( x \cdot \frac{e^x}{\cosh^2(e^x)} \), we consider its behavior as \( x \) approaches infinity and negative infinity:

   \[ \lim_{x \to \infty} \left| x \cdot \frac{e^x}{\cosh^2(e^x)} \right| = 1 \]
   \[ \lim_{x \to -\infty} \left| x \cdot \frac{e^x}{\cosh^2(e^x)} \right| = 0 \]

Since both limits are finite, the term \( x \cdot \frac{e^x}{\cosh^2(e^x)} \) is bounded on \( \mathbb{R} \).

Combining these findings, we conclude that \( |f'(x)| \) is bounded on \( \mathbb{R} \). The maximum value of \( |f'(x)| \) is \( 1 \), therefore we can take \( L = 1 \) as the Lipschitz constant.

Hence, \( f(x) = x \cdot \tanh(e^x) \) is Lipschitz continuous with a Lipschitz constant \( L = 1 \).
\end{proof}

Next, we show that TeLU has a smoother loss landscape, which leads to faster convergence.
\begin{theorem}\label{thm:fim_conv}
Given a neural network \( \mathcal{N} \) with activation function \( f(x) = x \cdot \tanh(e^x) \), parameters \( \theta \), and a differentiable loss function \( \mathcal{L}(\theta) \), the Fisher Information Matrix \( I(\theta) \) defined as
\[ I(\theta) = \mathbb{E}_{(x, y) \sim \mathcal{D}}\left[ \nabla_\theta \log p(y|x; \theta) \nabla_\theta \log p(y|x; \theta)^\top \right] \]
leads to a smoother optimization landscape during training of \( \mathcal{N} \).
\end{theorem}
\textit{Proof Sketch:} Based on prior results, we show the smoothness of TeLU and its derivative and how it leads to better Fisher information estimates \cite{fisher1925theory}. 
The detailed proof can be found in the appendix \ref{sec:TeLU_conv}

Finally we show with some mild assumption (Polyak-Łojasiewicz (PL) condition \cite{polyak1969minimization}) the global convergence of network trained using TeLU
\begin{theorem} \label{global_telu}
Let \( \mathcal{N} \) be a neural network employing the activation function \( f(x) = x \cdot \tanh(e^x) \) in its architecture. Assume the network parameters are denoted by \( \theta \) and the network is trained using a differentiable loss function \( \mathcal{L}(\theta) \). If \( \mathcal{L}(\theta) \) satisfies the Polyak-Łojasiewicz (PL) condition, then the gradient descent optimization on \( \mathcal{N} \) converges to a global minimum, significantly influenced by the properties of \( f(x) \) and it's derivative \( f'(x) \).
\end{theorem}
\textit{Proof Sketch:} We adapt this based on prior constructions, showing TeLU converges faster and has a smooth optimization curve and proving using PL condition that the network will converge to global optima. 
The detailed proof is shown in appendix \ref{sec:TeLU_conv}

Next, we empirically validate the effectiveness of the proposed TeLU activation function
\section{Experiments using TeLU}
This section presents a detailed assessment of the TeLU activation function implemented within deep neural architectures, specifically Squeezenet \cite{SqueezeNet} and Resnet-18/32/50 \cite{ResNet}. Our evaluation focuses on the stability and performance of TeLU across diverse optimization techniques, including Stochastic Gradient Descent (SGD) \cite{SGD1}, SGD with Momentum \cite{Momentum}, AdamW \cite{AdamW}. and RMSprop \cite{RMSprop}. We benchmark TeLU's effectiveness by comparing it with a range of established activation functions: (i) ReLU \cite{ReLU}, (ii) GELU \cite{GELU}, (iii) Mish \cite{Mish}, (iv) SiLU \cite{SiLU}, (v) Smish \cite{Smish}, and (vi) Logish \cite{Logish}.

\subsection{Datasets}
We utilized three benchmark datasets to evaluate our proposed model: CIFAR-10, CIFAR-100 \cite{CIFAR}, and TinyImageNet \cite{Tiny}. Each of these datasets is crucial for benchmarking the performance of image classification algorithms, especially Convolutional Neural Networks (CNNs).

\textbf{CIFAR-10}: This dataset comprises \(60,000\) color images of dimensions \(32 \times 32\) pixels, evenly distributed across \(10\) distinct classes. The dataset is partitioned into a training set of \(50,000\) images and a test set of \(10,000\) images. We split the dataset into \(45,000\) images for training, \(5,000\) images for validation, and \(10,000\) for testing.

\textbf{CIFAR-100}: Similar in size to CIFAR-10, CIFAR-100 contains \(60,000\) color images of \(32 \times 32\) pixels. However, it is differentiated by its finer categorization into \(100\) classes, with each class containing \(600\) images.  We split the dataset into \(45,000\) images for training, \(5,000\) images for validation, and \(10,000\) for testing.

\textbf{TinyImageNet}: As a subset of the larger ImageNet dataset, TinyImageNet includes \(110,000\) images resized to \(64 \times 64\) pixels. It spans \(200\) classes, with each class contributing \(500\) training images, \(50\) validation images, and \(50\) test images. We utilize the original training set of \(100,000\) images and validation set of \(10,000\) images, as no testing set is publicly available for TinyImageNet.

\subsection{Experimental Setup}
In our experimental framework, the activation function was the sole independent variable across all models, facilitating a focused analysis of its impact on model performance. These activation functions include TeLU, ReLU, GELU, Mish, SiLU, Smish, and Logish. We employed a comprehensive grid search methodology to meticulously optimize key hyperparameters – learning rate, learning rate decay (gamma), learning rate decay step size, and weight decay – thereby ensuring maximal accuracy on the validation subsets for a broad spectrum of activation function configurations. These hyperparameters were fine-tuned for each experimental setup, with their optimal values enumerated in the appendix for reference. We maintained a consistent batch size of $128$ across all trials to ensure uniformity in training conditions. For CIFAR-10 and CIFAR-100 experiments, the learning rate was decayed at epochs 60, 120, and 160. For TinyImageNet experiments, the learning rate steps occurred at 60, 100, 140, and 170. The optimal initial learning rate, learning rate decay gamma coefficient, and weight decay hyperparameters were identified based on their performance enhancement on the validation dataset. These tuned hyperparameters are detailed in supplementary tables: \ref{tab:CIFAR_sup_sq_hps}, \ref{tab:CIFAR_sup_r18_hps}, \ref{tab:CIFAR_sup_r34_hps}, \ref{tab:CIFAR_sup_r50_hps}, \ref{tab:CIFAR-100_sup_sq_hps}, \ref{tab:CIFAR-100_sup_r18_hps}, \ref{tab:CIFAR-100_sup_r34_hps}, \ref{tab:CIFAR-100_sup_r50_hps}, \ref{tab:tiny_sup_r34_hps}.  Each experiment was conducted over $200$ epochs per model, and these were replicated across 5 distinct trials to guarantee statistical robustness. Our experimental matrix was extensive, encompassing a diverse array of datasets (CIFAR-10, CIFAR-100, and TinyImageNet), neural network architectures (SqueezeNet, ResNet18, ResNet34, and ResNet50), and optimization algorithms (SGD, SGD with Momentum, AdamW, and RMSprop). It is noteworthy that for experiments involving the TinyImageNet dataset, we exclusively utilized the ResNet34 architecture due to computational limits. 

\begin{figure}
    \centering
    \includegraphics[width=0.75\linewidth]{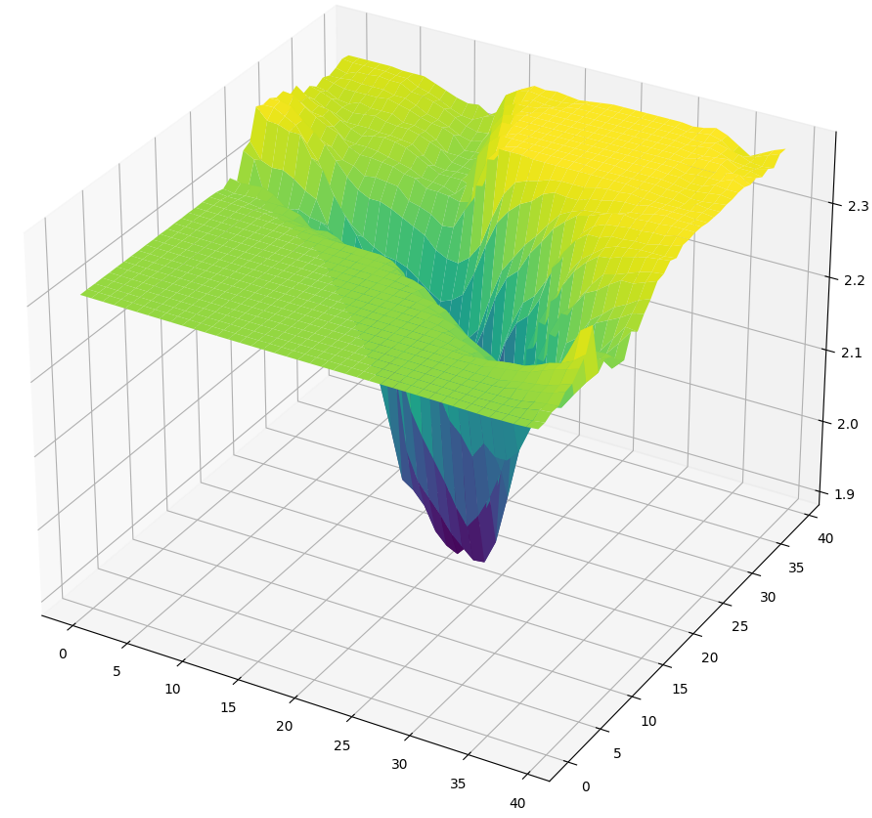}
    \caption{ReLU Loss Landscape}
    \label{fig:loss_relu}
\end{figure}

\begin{figure}
    \centering
    \includegraphics[width=0.75\linewidth]{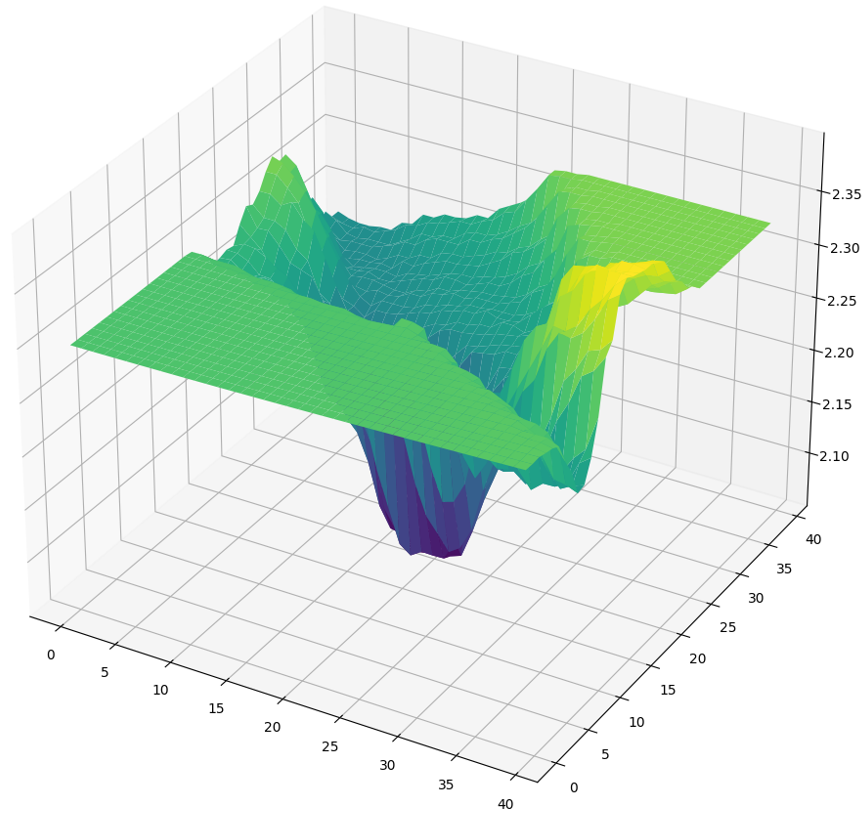}
    \caption{TeLU Loss Landscape}
    \label{fig:loss_sish}
\end{figure}
\subsection{CIFAR-10 Experiments}

The primary objective of these experiments was to rigorously evaluate the generalization efficacy of various activation functions within the context of complex, natural image datasets. Table \ref{tab:CIFARSUM} presents a comparative analysis of different activation functions applied to the Squeezenet architecture on the CIFAR-10 dataset. The results delineated in Table \ref{tab:CIFARSUM} clearly demonstrate that the TeLU activation function consistently surpasses its counterparts in most scenarios, not only in terms of performance but also by exhibiting a notably lower variance.

\begin{figure}
    \centering
    \includegraphics[width=0.75\linewidth]{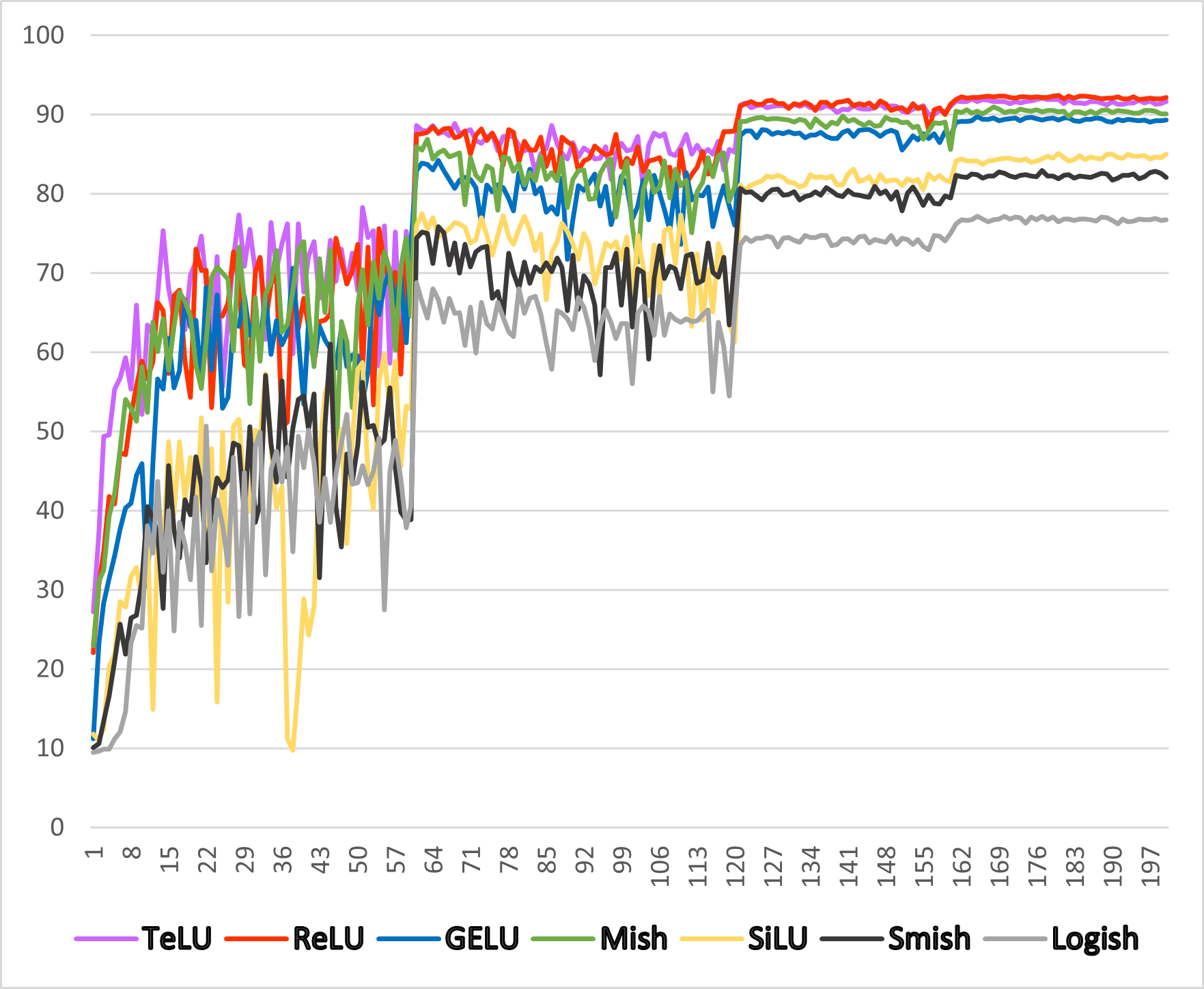}
    \caption{Validation performance comparison of 7 activation functions per epoch on CIFAR-10 using SqueezeNet-SGD}
    \label{fig:sgd_CIFAR-10}
\end{figure}

For instance, a comparative case involving Logish highlights its underperformance, particularly when trained using SGD, where it exhibits a significant variance of 29. It's crucial to acknowledge that while each network was meticulously optimized for each optimizer, Logish achieved a peak accuracy of 90\% on one seed but exhibited marked performance degradation on others.
Furthermore, a close examination reveals that ReLU, albeit being the second most effective activation function in our study, experiences a performance decline of 3.25\% when transitioning from SGD to RMSProp. In stark contrast, TeLU maintains robustness across optimizers, evidenced by the smallest average performance drop of merely 1.84\%.  This is also evident in Figure \ref{fig:sgd_CIFAR-10} and \ref{fig:momentum_CIFAR-10} we show per epoch validation curve for each activation function for a trial. This underscores TeLU's superior adaptability and stability across different optimization environments. In figure \ref{fig:loss_relu} and \ref{fig:loss_sish}, we plot the 3D loss landscape surface for both ReLU and TeLU, respectively, thus validating our theoretical findings. We observe a similar trend for the other 3 architectures; we report the results in appendix \ref{sec:sup_res}.

\begin{table}[htb!]
\centering
\caption{CIFAR-10 SqueezeNet Test Accuracy Summary}
{
\begin{tabular}{||c c c c c||} 
 \hline
 Name & SGD & Momentum & AdamW & RMSprop\\
 \hline\hline
 TeLU & 91.40$\pm$ \textbf{0.11} & \textbf{90.96}$\pm$0.29 & \textbf{90.08}$\pm$0.77 & \textbf{89.86}$\pm$0.28\\ 
 \hline
 ReLU & \textbf{91.84}$\pm$0.33 & 90.77$\pm$\textbf{0.16} & 89.01$\pm$\textbf{0.45} & 88.59$\pm$\textbf{0.14}\\ 
 \hline
 GELU & 88.42$\pm$0.28 & 89.33$\pm$0.24 & 89.63$\pm$0.70 & 80.68$\pm$1.2 \\
 \hline
 Mish & 89.87$\pm$0.21 & 90.04$\pm$0.25 & 89.02$\pm$87 & 87.39$\pm$0.17\\
 \hline
 SiLU & 78.61$\pm$6.3 & 84.10$\pm$1.1 & 86.70$\pm$1.9 & 66.00$\pm$1.3\\
 \hline
 Smish & 77.28$\pm$3.0 & 68.60$\pm$2.2 & 41.71$\pm$17 & 66.91$\pm$2.3\\
 \hline
 Logish & 61.44$\pm$29 & 66.10$\pm$3.6 & 42.72$\pm$16 & 43.20$\pm$19\\ [1ex] 
 \hline
\end{tabular}}
\label{tab:CIFARSUM}
\end{table}

\begin{figure}
    \centering
    \includegraphics[width=0.75\linewidth]{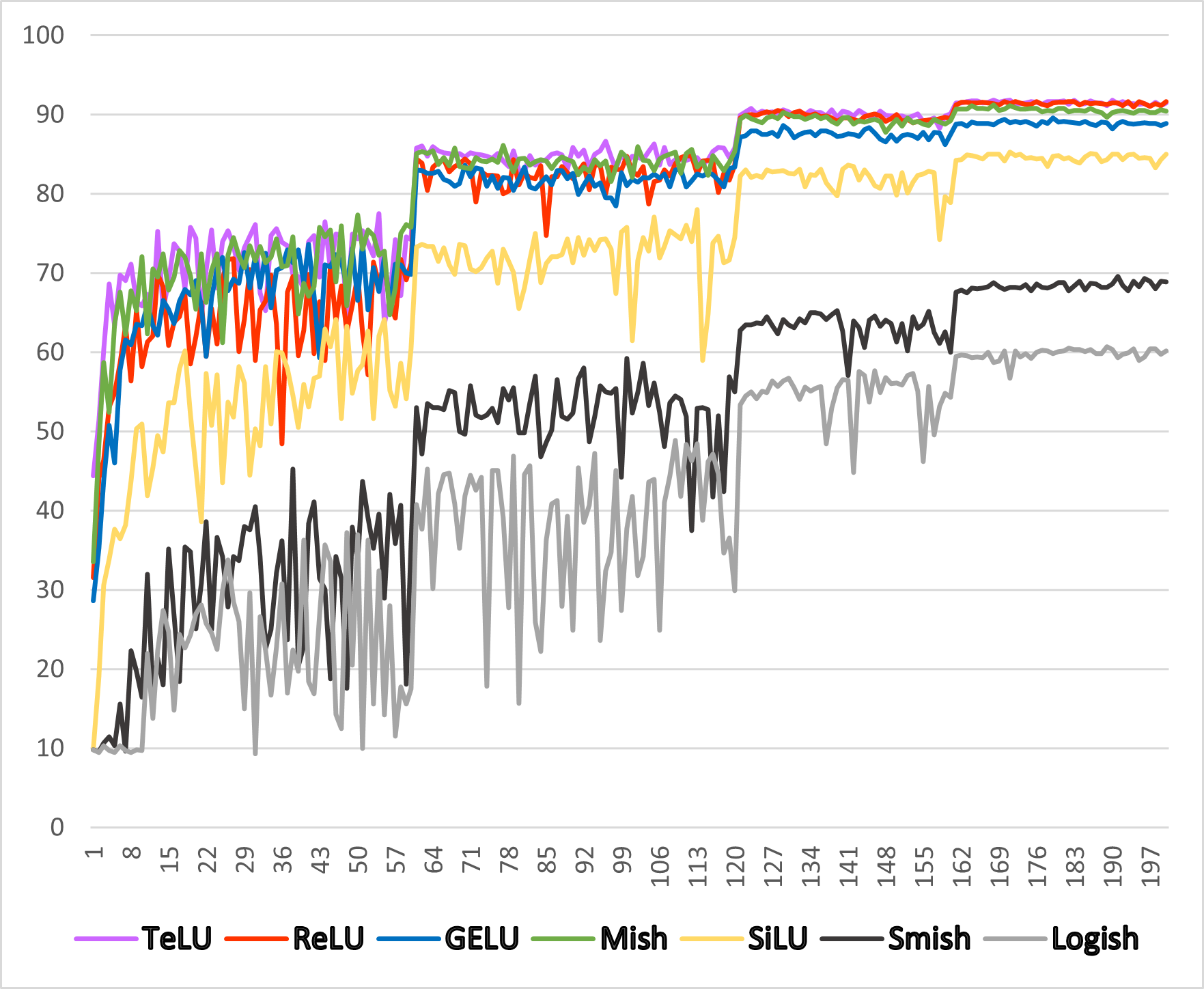}
    \caption{Validation performance comparison of 7 activation functions per epoch on CIFAR-10 using SqueezeNet Momentum}
    \label{fig:momentum_CIFAR-10}
\end{figure}

\begin{figure}
    \centering
    \includegraphics[width=0.8\linewidth]{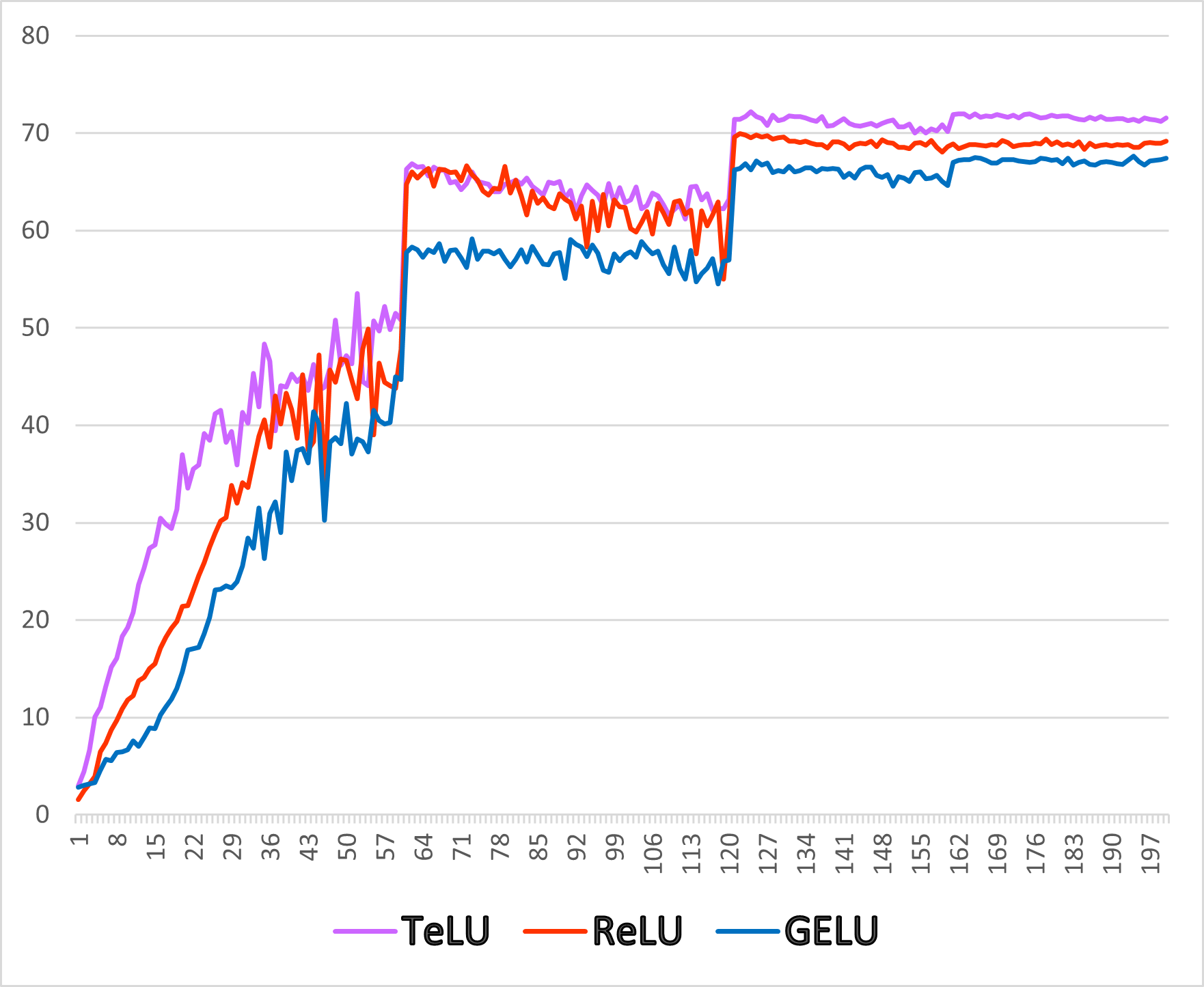}
    \caption{Validation performance comparison of TeLU, ReLU and GELU per epoch on CIFAR-100 using SqueezeNet SGD }
    \label{fig:squeezenet_CIFAR-100}
\end{figure}
In this section, we extend our analysis to the CIFAR-100 benchmark, focusing on evaluating the robustness of our TeLU (Hyperbolic Tangent Exponential Linear Unit) activated model in extracting intricate features and its resilience against overfitting to specific class attributes. The intrinsic regularization properties of TeLU contribute to its reduced overfitting tendencies when compared to ReLU, which, in our experimental setup, displayed comparable performance to TeLU. Our prior investigations revealed a significant similarity in the hyperparameter landscape for TeLU, ReLU, and GELU, in contrast to the other four evaluated activation functions. This similarity facilitates a more streamlined and efficient hyperparameter optimization process.
Building on the preliminary findings, which indicated a propensity for larger variance in other activation functions, we confined our subsequent experiments to the top-performing trio of activation functions (TeLU, ReLU, and GELU). This phase involved a comprehensive evaluation across four different architectural frameworks, employing four distinct optimization algorithms. The comparative results are meticulously detailed in Table \ref{tab:CIFAR100SUM}, where TeLU's consistent top-tier performance across various optimizers is underscored alongside its characteristic lower variance profile. 

\begin{table}[htb!]
\centering
\caption{CIFAR-100 SqueezeNet Test Accuracy Summary}
    \begin{tabular}{||c c c c c||}  
 \hline
 Name & SGD & Momentum & AdamW & RMSprop\\
 \hline\hline
 TeLU & \textbf{71.47}$\pm$\textbf{0.08} & \textbf{70.53}$\pm$\textbf{0.25} & \textbf{69.64}$\pm$\textbf{0.07} & \textbf{68.83}$\pm$0.33\\ 
 \hline
 ReLU & 69.52$\pm$0.43 & 65.05$\pm$0.51 & 66.31$\pm$0.48 & 67.99$\pm$\textbf{0.21}\\ 
 \hline
 GELU & 67.09$\pm$0.36 & 66.26$\pm$29 & 66.50$\pm$0.44 & 65.19$\pm$0.25 \\ [1ex]
 \hline
\end{tabular}
\label{tab:CIFAR100SUM}
\end{table}
\subsection{CIFAR-100 Experiments}

\begin{figure}
    \centering
    \includegraphics[width=0.75\linewidth]{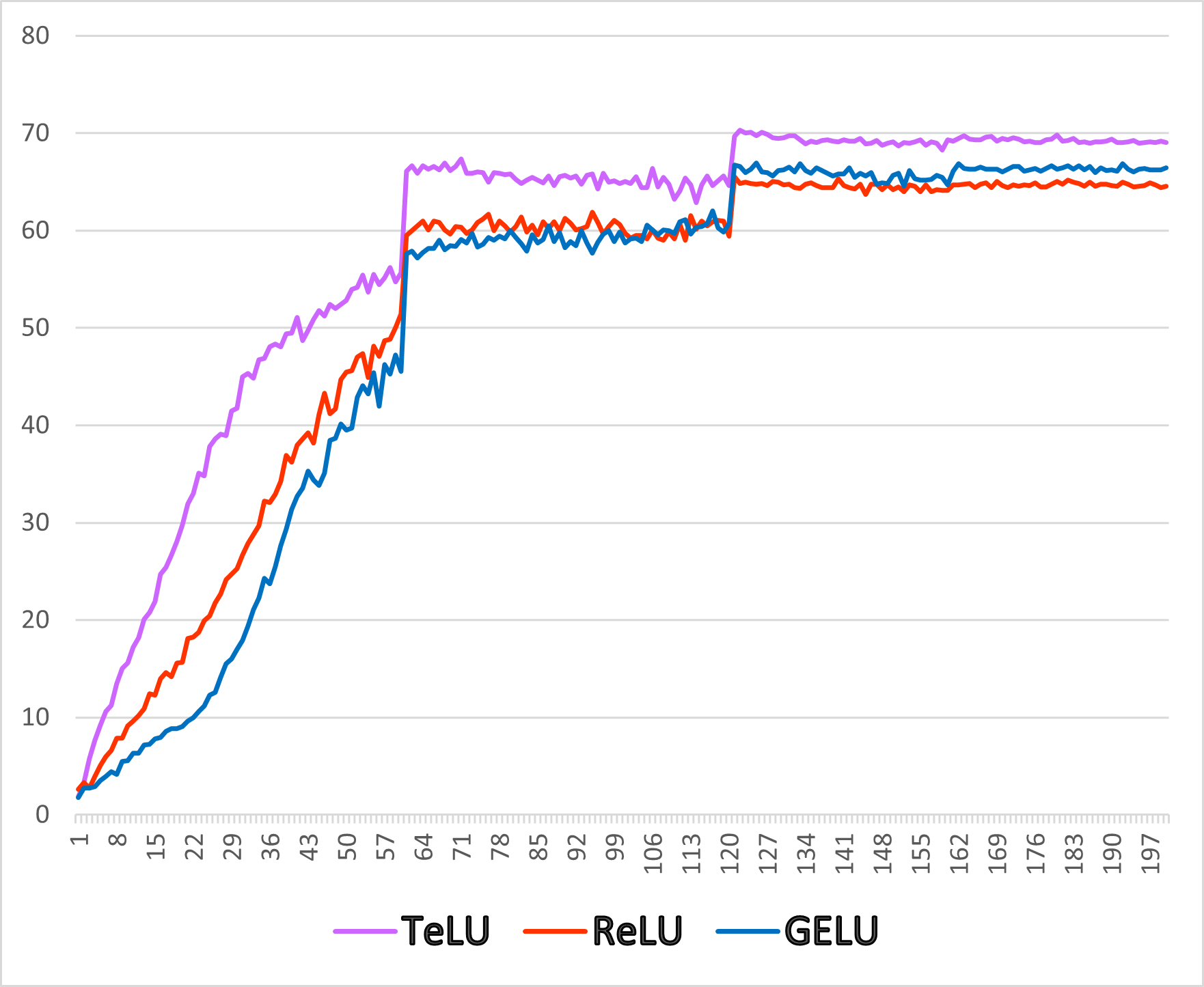}
    \caption{Validation performance comparison of TeLU, ReLU and GELU per epoch on CIFAR-100 using Momentum}
    \label{fig:squeezenet_mom}
\end{figure}

The empirical outcomes are further elucidated through Figures \ref{fig:squeezenet_CIFAR-100} and \ref{fig:squeezenet_mom}, which depict the validation performance of models employing Squeezenet architecture and trained using SGD and Momentum optimizers, respectively. These visual representations clearly demonstrate TeLU's superior convergence rate relative to ReLU and GELU, ultimately leading to more optimal solutions. This enhanced convergence efficiency of TeLU is particularly notable in the context of complex datasets like CIFAR-100, reinforcing its potential as a highly effective activation function in advanced neural network applications. In the appendix\ref{sec:sup_res}, we report performance for the remaining 3 architectures, where a similar trend was observed.

\begin{figure}
    \centering
    \includegraphics[width=0.75\linewidth]{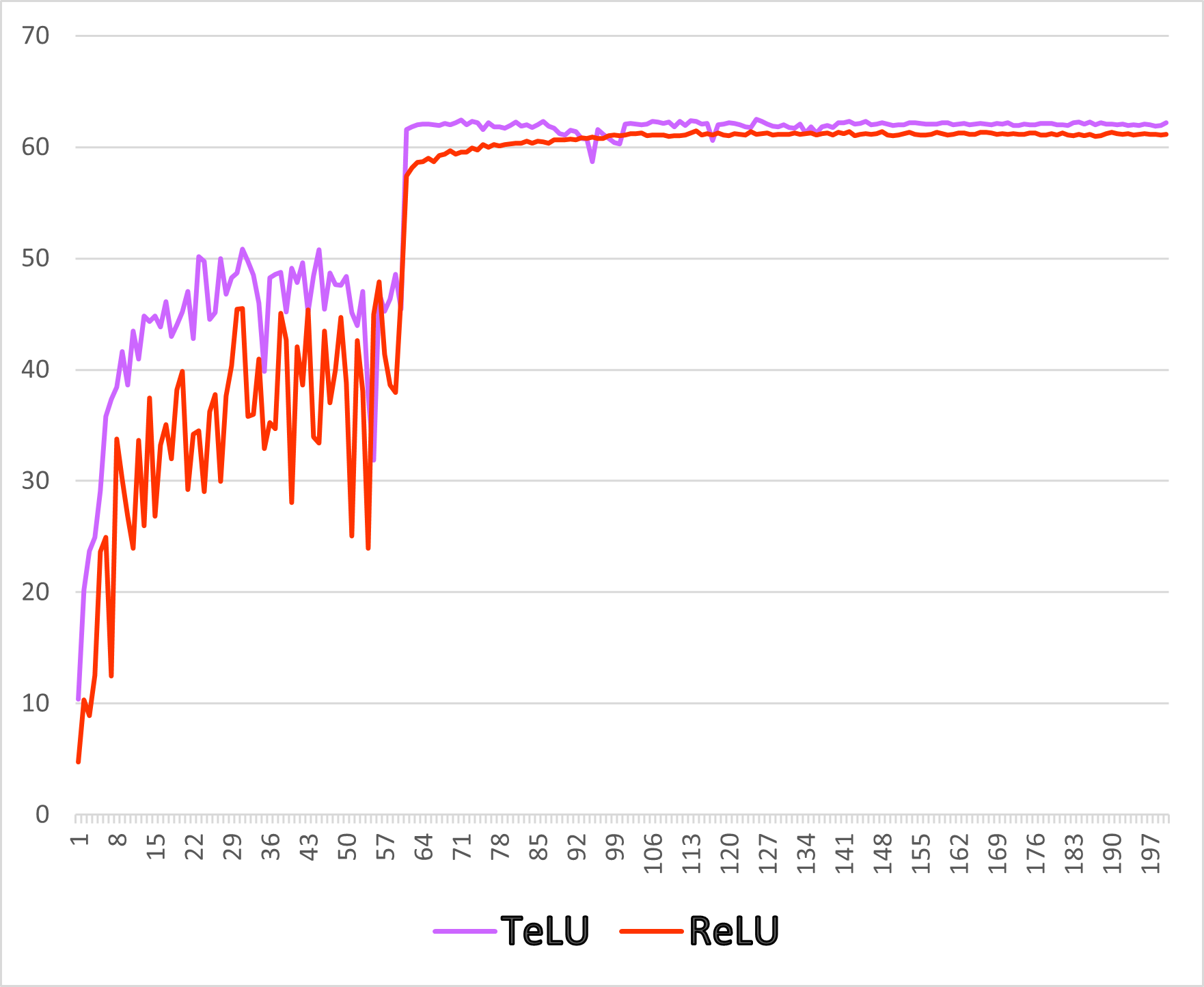}
    \caption{Validation performance comparison of TeLU, ReLU per epoch on TinyImagenet using Resnet-34 SGD. }
    \label{fig:tinyimg_sgd}
\end{figure}

\subsection{TinyImageNet200 Experiments}
In this detailed analysis, we probe the hierarchical representation learning capabilities of the TeLU (Hyperbolic Tangent Exponential Linear Unit) activation function within high-dimensional, complex imagery contexts, employing the Resnet-34 architecture—a model noted for its depth and complexity. Given that our preceding analysis positioned TeLU and ReLU as the leading activation functions, we conducted a focused evaluation using the TinyImagenet benchmark to compare their performance intricacies.
The results, systematically tabulated in Table \ref{tab:tinyimgSUM}, reveal a consistent outperformance by TeLU over ReLU. A particularly intriguing observation is the marked inconsistency of ReLU under Momentum-based training. We noted that while ReLU achieved an accuracy peak of nearly 64\% for one specific seed, its performance plummeted to below 20\% for other seeds, resulting in an extraordinarily high variance of 34\%. This variability is a critical indicator of ReLU's instability under certain training conditions. Figure \ref{fig:tinyimg_mom} presents the maximum, mean, and minimum performance metrics for both TeLU and ReLU to visually encapsulate and further scrutinize this instability. This graphical representation provides a clear and comprehensive view of the performance disparities between the two activation functions. Additionally, Figure \ref{fig:tinyimg_sgd} focuses on the models trained using the Stochastic Gradient Descent (SGD) optimizer, where ReLU demonstrates a more stable behavior. Despite this stability, it is noteworthy that TeLU exhibits a significantly accelerated convergence rate even in this scenario compared to ReLU, indicating its efficiency in navigating toward optimal solutions more rapidly. This aspect is particularly critical in deep learning models where time-to-convergence is vital in evaluating the effectiveness of activation functions.

\begin{table}[]
\centering
 \caption{TinyImageNet ResNet34 Test Accuracy Summary.}
    \begin{tabular}{||c c c c c||} 
 \hline
 Name & SGD & Momentum & AdamW & RMSprop\\
 \hline\hline
 TeLU & \textbf{62.34}$\pm$\textbf{0.17} & \textbf{62.09}$\pm$\textbf{0.22} & 54.04$\pm$0.82 & \textbf{58.48}$\pm$\textbf{0.03}\\ 
 \hline
 ReLU & 61.16$\pm$0.31 & 38.37$\pm$34 & \textbf{54.88}$\pm$\textbf{0.72} & 58.33$\pm$0.27\\ [1ex]
 \hline
\end{tabular}
\label{tab:tinyimgSUM}
\end{table}

\begin{figure}
    \centering
    \includegraphics[width=0.75\linewidth]{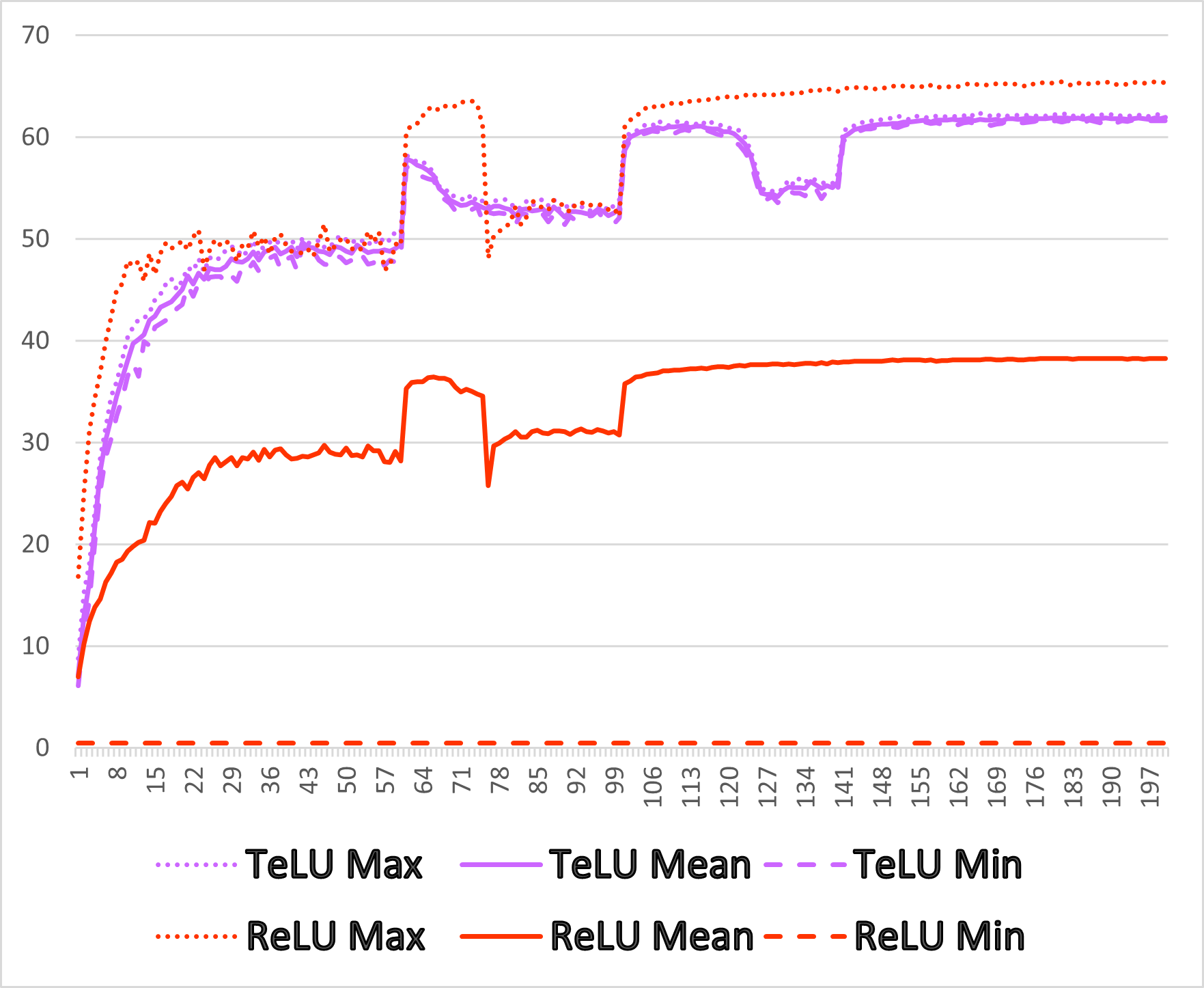}
    \caption{Validation accuracy per epoch on TinyImageNet using ResNet34-Momentum. TeLU and ReLU Validation curves are shown for best, mean, and minimum scenarios across 5 trials.}
    \label{fig:tinyimg_mom}
\end{figure}

\section{Conclusion}
In this work, we have successfully introduced the TeLU, a novel activation function designed to catalyze stable, efficient, and robust learning in deep neural networks. TeLU is Lipschitz continuous and saturates towards large negative value. For symmetric probability distribution, TeLU shifts the activation mean towards zero, which aligns gradients more closely with unit natural gradients, thereby accelerating convergence and introducing stability. Furthermore, TeLU's controlled growth for positive inputs and its saturating behavior for negative values underscore its robustness and stability – vital attributes for reliable neural network performance. Empirical evidence strongly supports TeLU's superiority. Across three major vision benchmarks, TeLU consistently outshines other activation functions. It exhibits remarkable stability across various experimental conditions, starkly contrasting the often unstable behaviors observed with ReLU and GELU under similar circumstances. TeLU's consistency across different optimization strategies is particularly noteworthy, reaching near-uniform conclusions and exhibiting minimal variance in performance.

\section{Impact Statement}
In this work, we have introduced a novel activation function, poised to improve neural network training with properties such as theoretical stability, rapid convergence, and enhanced robustness. This innovative approach is a positive step towards an efficient neural network, which promises a positive direction in significantly reducing energy consumption, a vital step towards more sustainable and environmentally friendly AI technologies. By focusing on creating more efficient models, we are paving the way for a future where advanced deep learning can be both high-performing and energy-conscious. While our contribution marks a significant advancement in technical aspects of neural network design, we acknowledge that it does not directly address the broader social, ethical, fairness, and bias challenges inherent in deep learning architectures. These issues require a holistic approach, combining technical innovation with rigorous ethical standards and inclusive practices to ensure AI is fair and beneficial for all.


\bibliography{main_arxiv}
\bibliographystyle{acm}

\appendix
\onecolumn
\section{Robustness comparison of TeLU with other activations }
\label{sec:app_robust}

We compare the robustness of Mish, GELU, ELU \cite{elu}, and \( f(x) = x \cdot \tanh(e^x) \) functions by examining and comparing their derivatives.

1. \textbf{Mish Function:}
   \[ \text{Mish}(x) = x \cdot \tanh(\ln(1 + e^x)) \]
   The derivative is complex, involving the derivative of tanh and the exponential function.

2. \textbf{GELU Function:}
   \[ \text{GELU}(x) \approx 0.5x \left(1 + \tanh\left[\sqrt{2/\pi} \left(x + 0.044715x^3\right)\right]\right) \]
   The derivative involves both tanh and polynomial components.

3. \textbf{ELU Function:}
   \[ \text{ELU}(x) = \begin{cases} x & \text{if } x > 0 \\ \alpha(e^x - 1) & \text{if } x \leq 0 \end{cases} \]
   \[ \text{ELU}'(x) = \begin{cases} 1 & \text{if } x > 0 \\ \alpha e^x & \text{if } x \leq 0 \end{cases} \]

4. \textbf{\( f(x) = x \cdot \tanh(e^x) \) Function:}
   \[ f(x) = x \cdot \tanh(e^x) \]
   \[ f'(x) = \tanh(e^x) + x \cdot (1 - \tanh^2(e^x)) \cdot e^x \]

The robustness of these functions to small input perturbations can be inferred from the behavior of their derivatives. A large derivative in magnitude or varies rapidly with respect to \( x \) indicates less robustness to small changes in input. In contrast, a derivative that remains bounded and changes smoothly suggests greater robustness.

Based on this criterion, we can qualitatively rank the robustness of these functions, which ranks TeLU first, followed by GELU, ELU, and then Mish.

\section{Convergence Guarantee of TeLU}
\label{sec:TeLU_conv}
First, we show ReLU doesn't have mean shifting property and doesn't exhibit a regularization effect
\begin{theorem}
   The Rectified Linear Unit (ReLU) function, defined as \( \text{ReLU}(x) = \max(0, x) \), does not exhibit mean-shifting capability over symmetric intervals \( [-a, a] \) around zero 
\end{theorem} \label{thm:relu}
\begin{proof}
\begin{enumerate}
\item \(\text{ReLU}(x) = \max(0, x)\). This implies \(\text{ReLU}(x) = 0\) for \(x \leq 0\) and \(\text{ReLU}(x) = x\) for \(x > 0\).
\item Consider \(I(a) = \int_{-a}^{a} \text{ReLU}(x) \, dx\).
\item Splitting the integral:
\[ I(a) = \int_{-a}^{0} \text{ReLU}(x) \, dx + \int_{0}^{a} \text{ReLU}(x) \, dx \]
\item Evaluating the integrals:
\[ I(a) = 0 + \int_{0}^{a} x \, dx = 0 + \frac{a^2}{2} \]
\item Average value over \( [-a, a] \):
\[ \frac{1}{2a} I(a) = \frac{a}{4} \]
\item As \(a\) increases, the average value increases, not approaching zero.
\end{enumerate}
This concludes the proof
\end{proof}

Next, we show that TeLU ($f(x)$) and its derivative ($f'(x)$) are both continuous, and that this condition is true even based on the Intermediate Value Theorem (IVT) \cite{bolzano1905rein} and the Mean Value Theorem (MVT) \cite{cauchy1821cours}.

\begin{theorem}
Let $f(x) = x \cdot \tanh(e^x)$ be defined for all $x \in \mathbb{R}$. Then:
\begin{enumerate}
    \item The function $f(x)$ and its derivative $f'(x)$ are continuous for all $x \in \mathbb{R}$.
    \item The function $f(x)$ satisfies the Intermediate Value Theorem (IVT) on any interval $[a, b] \subset \mathbb{R}$.
    \item The function $f(x)$ satisfies the Mean Value Theorem (MVT) on any interval $[a, b] \subset \mathbb{R}$, where $a \neq b$.
\end{enumerate}
\end{theorem}

\begin{proof}
    \subsection*{Continuity of $f(x)$ and $f'(x)$}
\begin{itemize}
    \item $f(x) = x \cdot \tanh(e^x)$ is continuous as both $x$ and $\tanh(e^x)$ are continuous.
    \item The derivative $f'(x) = \tanh(e^x) + x \cdot sech^2(e^x) \cdot e^x$ is continuous since $\tanh(e^x)$, $sech^2(e^x)$, and $e^x$ are continuous.
\end{itemize}

\subsection*{Application of the IVT}
The Intermediate Value Theorem \cite{bolzano1905rein} states that if a function is continuous on a closed interval, then it takes on every value between its values at the endpoints of the interval.
\begin{itemize}
    \item For $f(x)$ on any interval $[a, b]$, if $d$ is a value between $f(a)$ and $f(b)$, there exists a $c \in [a, b]$ such that $f(c) = d$.
    \item This is because $f(x)$ is continuous on $[a, b]$.
    \item Similarly, since $f'(x)$ is continuous on any interval $[a, b]$, by IVT, for any value $e$ between $f'(a)$ and $f'(b)$, there exists a $c \in [a, b]$ such that $f'(c) = e$.
\end{itemize}

\subsection*{Application of the MVT}
The Mean Value Theorem \cite{cauchy1821cours} states that if a function is continuous on a closed interval and differentiable on the open interval, then there exists at least one point in the open interval where the derivative equals the average rate of change over the closed interval.
\begin{itemize}
    \item Since $f(x)$ is continuous on $[a, b]$ and $f'(x)$ is continuous on $(a, b)$, by MVT, there exists at least one $c \in (a, b)$ such that $f'(c) = \frac{f(b) - f(a)}{b - a}$.
\end{itemize}


The function $f(x) = x \cdot \tanh(e^x)$ and its derivative are continuous, and $f(x)$ and $f'(x)$ satisfies both the IVT and MVT on any interval in $\mathbb{R}$.

This concludes the proof
\end{proof}

\textbf{Next we provided detailed proof for Theorem \ref{thm:fim_conv} discussed in main paper}
\begin{proof}


\textbf{Continuity and Differentiability of \( f(x) \)}

The activation function \( f(x) = x \cdot \tanh(e^x) \) and its derivative are analyzed:
\begin{align*}
    f(x) &= x \cdot \tanh(e^x), \\
    \text{where } \tanh(u) &= \frac{e^{2u} - 1}{e^{2u} + 1}. \\
    \text{Thus, } f'(x) &= \frac{d}{dx}(x \cdot \tanh(e^x)) \\
    &= \tanh(e^x) + x \cdot \frac{d}{dx} \tanh(e^x) \\
    &= \tanh(e^x) + x \cdot e^x \cdot (1 - \tanh^2(e^x)).
\end{align*}
Since \( \tanh(u) \) and \( e^x \) are continuously differentiable, \( f(x) \) and \( f'(x) \) are also continuously differentiable.

\textbf{Impact on Fisher Information Matrix}

Applying the chain rule to compute the gradient of the log-likelihood:
\begin{align*}
    \nabla_\theta \log p(y|x; \theta) &= \frac{\partial \log p(y|x; \theta)}{\partial \mathcal{N}} \cdot \frac{\partial \mathcal{N}}{\partial \theta}, \\
    &= \text{Gradient of the output w.r.t. the network's parameters}.
\end{align*}
The gradient involves terms from \( f'(x) \) due to the activation function in each layer:
\begin{align*}
    f'(x) &= \tanh(e^x) + x \cdot e^x \cdot (1 - \tanh^2(e^x)).
\end{align*}
Thus, \( I(\theta) \) becomes a matrix of expectations of outer products of these gradients:
\begin{align*}
    I(\theta) &= \mathbb{E}_{(x, y) \sim \mathcal{D}}\left[ \nabla_\theta \log p(y|x; \theta) \nabla_\theta \log p(y|x; \theta)^\top \right].
\end{align*}
The smoothness of \( f'(x) \) translates to a smoother \( I(\theta) \).

\textbf{Smoother Optimization Landscape}

In gradient descent, parameter updates are governed by:
\begin{align*}
    \theta^{(t+1)} &= \theta^{(t)} - \eta \cdot \nabla_\theta \mathcal{L}(\theta^{(t)}),
\end{align*}
where \( \eta \) is the learning rate. The gradient of the loss function \( \nabla_\theta \mathcal{L}(\theta) \) is influenced by \( I(\theta) \). A smoother \( I(\theta) \) results in more stable and consistent gradient updates, avoiding erratic steps often observed in rougher optimization landscapes. This leads to enhanced stability in finding the minima of \( \mathcal{L}(\theta) \).

Hence, we can show, that the continuously differentiable nature of \( f(x) = x \cdot \tanh(e^x) \) and its derivative ensures that the Fisher Information Matrix \( I(\theta) \) in the neural network \( \mathcal{N} \) promotes a smoother optimization landscape, facilitating more effective training dynamics.

\end{proof}

Based on the properties of  TeLU, shown in Theorem \ref{thm:fim_conv}, we can prove the global convergence of the function under certain conditions.

\textbf{Now we provided detailed proof for Theorem \ref{global_telu} discussed in main paper}
\begin{proof}

   \textit{Smoothness and Boundedness of \( f(x) \) and \( f'(x) \):}
   
    The function \( f(x) = x \cdot \tanh(e^x) \) is continuously differentiable. Its derivative, given by
    \[ f'(x) = \tanh(e^x) + x \cdot e^x \cdot (1 - \tanh^2(e^x)), \]
    is also continuously differentiable and bounded due to the inherent properties of the \( \tanh \) function and the exponential function. These properties ensure smooth and well-conditioned gradient computations throughout the optimization process.

   \textit{Influence on Gradient Descent under PL Condition:}
   
    Given the PL condition, for a global minimum \( \theta^* \), there exists \( \mu > 0 \) such that
    \[ 2\mu(\mathcal{L}(\theta) - \mathcal{L}(\theta^*)) \leq \|\nabla_\theta \mathcal{L}(\theta)\|^2 \text{ for all } \theta. \]
    The gradient descent update rule is
    \[ \theta^{(t+1)} = \theta^{(t)} - \eta \cdot \nabla_\theta \mathcal{L}(\theta^{(t)}), \]
    where \( \eta \) is the learning rate. 

    \textit{Convergence Analysis:}
    
    Utilizing the smoothness and boundedness of \( f'(x) \), along with the PL condition, it can be shown that
    \[ \mathcal{L}(\theta^{(t+1)}) \leq \mathcal{L}(\theta^{(t)}) - \eta \cdot \|\nabla_\theta \mathcal{L}(\theta^{(t)})\|^2, \]
    which implies
    \[ \mathcal{L}(\theta^{(t)}) - \mathcal{L}(\theta^*) \leq \left(1 - 2\mu\eta\right)^t (\mathcal{L}(\theta^{(0)}) - \mathcal{L}(\theta^*)). \]
    Therefore, \( \mathcal{L}(\theta^{(t)}) \) converges to \( \mathcal{L}(\theta^*) \) as \( t \rightarrow \infty \).



\end{proof}

\section{Supplementary Results}
\label{sec:sup_res}
In this study, we adopted the hyperparameter ranges from existing research on ReLU and Mish as a baseline, conducting a comprehensive grid search within these parameters for all evaluated activation functions. This experimental approach, particularly focused on very deep architectures, was strategically designed to assess whether TeLU could identify more optimal hyperparameters, even under these complex conditions. Our findings consistently demonstrate that TeLU secures a position within the top three performers, regardless of the optimizer configuration or architectural framework in use. This result underscores TeLU's adaptability and effectiveness in diverse neural network environments. It is important to note that parameters were not optimized only to favor TeLU but were designed to favor most activation functions after analyzing their validation performance for the first $30$ epochs.The Tables \ref{tab:CIFAR_sup_sq_hps}, \ref{tab:CIFAR_sup_r18_hps}, \ref{tab:CIFAR_sup_r34_hps}, \ref{tab:CIFAR_sup_r50_hps} provide the best hyperparameter settings on CIFAR-10 for squeeznet, Resnet-18, resnet-32, and resnet-50 architectures respectively. Tables \textbf{8 to 23} shows the average performance of the model across 5 trials for all combinations on CIFAR-10. It is evident from our comprehensive experiment that TeLU stays consistent throughput and stays in top-3 in terms of performance and stability. In terms of convergence label as \textbf{conc} in our tables, all activation functions achieve best performance in similar range, however TeLU stays consistent when it reaches best performance. Meanwhile, others become unstable, or performance drops after a few epochs before gaining momentum. Similarly, we report the best hyperparameters for all the architectures on CIFAR-100 in Tables \textbf{24 to 27}, and Tables \textbf{26 to 43} show the average performance of TeLU compared to ReLU and GELU, where TeLU consistently outperforms other activation in majority of the scenarios and with good stability.  Finally, Table \ref{tab:tiny_sup_r34_hps} shows the best hyperparameters used in the experiment, and tables \textbf{45 to 48} shows the performance of TeLU, which consistently outperforms ReLU both in performance and stability.


\begin{table}[]
    \centering
    \caption{CIFAR-10 SqueezeNet Hyperparameters}
    \label{tab:CIFAR_sup_sq_hps}
    \begin{tabular}{||c c c c||} 
 \hline
 Optimizer & learning rate & weight decay & gamma\\
 \hline\hline
 SGD & 0.1 & 0.003 & 0.2\\ 
 \hline
 Momentum & 0.1 & 0.0007 & 0.2\\ 
 \hline
 AdamW & 0.005 & 0.005 & 0.4\\ 
 \hline
 RMSprop & 0.0002 & 0.005 & 0.4\\  [1ex]
 \hline
\end{tabular}
\end{table}

\begin{table}[]
    \centering
    \caption{CIFAR-10 ResNet18 Hyperparameters}
    \label{tab:CIFAR_sup_r18_hps}
    \begin{tabular}{||c c c c||} 
 \hline
 Optimizer & learning rate & weight decay & gamma\\
 \hline\hline
 SGD & 0.1 & 0.0003 & 0.3\\ 
 \hline
 Momentum & 0.1 & 0.0002 & 0.2\\ 
 \hline
 AdamW & 0.008 & 0.007 & 0.5\\ 
 \hline
 RMSprop & 0.0005 & 0.0005 & 0.2\\  [1ex]
 \hline
\end{tabular}
\end{table}

\begin{table}[]
    \centering
    \caption{CIFAR-10 ResNet34 Hyperparameters}
    \label{tab:CIFAR_sup_r34_hps}
    \begin{tabular}{||c c c c||} 
 \hline
 Optimizer & learning rate & weight decay & gamma\\
 \hline\hline
 SGD & 0.01 & 0.003 & 0.5\\ 
 \hline
 Momentum & 0.01 & 0.001 & 0.5\\ 
 \hline
 AdamW & 0.001 & 0.005 & 0.5\\ 
 \hline
 RMSprop & 0.0001 & 0.001 & 0.5\\  [1ex]
 \hline
\end{tabular}
\end{table}

\begin{table}[]
    \centering
    \caption{CIFAR-10 ResNet50 Hyperparameters}
    \label{tab:CIFAR_sup_r50_hps}
    \begin{tabular}{||c c c c||} 
 \hline
 Optimizer & learning rate & weight decay & gamma\\
 \hline\hline
 SGD & 0.01 & 0.003 & 0.5\\ 
 \hline
 Momentum & 0.01 & 0.001 & 0.5\\ 
 \hline
 AdamW & 0.001 & 0.005 & 0.5\\ 
 \hline
 RMSprop & 0.0001 & 0.001 & 0.5\\  [1ex]
 \hline
\end{tabular}
\end{table}

\begin{table}[]
    \centering
    \caption{CIFAR-10 SqueezeNet SGD}
    \label{tab:CIFAR_sup_snet_sgd}
    \begin{tabular}{||c c c c c||} 
 \hline
 Name & Train & Valid & Test & Conc\\
 \hline\hline
 TeLU & 96.59 & 92.00 & 91.40$\pm$\textbf{0.106} & 91.60\\ 
 \hline
 ReLU & \textbf{99.17} & \textbf{92.74} & \textbf{91.84}$\pm$0.331 & 92.39\\ 
 \hline
 GELU & 92.39 & 89.42 & 88.42$\pm$0.280 & 88.93 \\
 \hline
 Mish & 94.23 & 90.83 & 89.87$\pm$0.213 & 90.21\\
 \hline
 SiLU & 79.36 & 79.99 & 78.61$\pm$6.313 & 67.09\\
 \hline
 Smish & 77.37 & 78.56 & 77.28$\pm$3.000 & 78.10\\
 \hline
 Logish & 61.35 & 62.37 & 61.44$\pm$29.28 & 61.868\\ [1ex] 
 \hline
\end{tabular}
\end{table}

\begin{table}[]
    \centering
    \caption{CIFAR-10 SqueezeNet Momentum}
    \label{tab:CIFAR_sup_snet_mom}
    \begin{tabular}{||c c c c c||} 
 \hline
 Name & Train & Valid & Test & Conc\\
 \hline\hline
 TeLU & 95.71 & 91.49 & \textbf{90.96}$\pm$0.290 & 90.88\\ 
 \hline
 ReLU & \textbf{98.55} & \textbf{91.66} & 90.77$\pm$\textbf{0.165} & 91.38\\ 
 \hline
 GELU & 94.70 & 90.33 & 89.33$\pm$0.243 & 89.78\\
 \hline
 Mish & 95.30 & 90.75 & 90.04$\pm$0.251 & 90.12\\
 \hline
 SiLU & 86.53 & 85.50 & 84.10$\pm$1.059 & 85.18\\
 \hline
 Smish & 68.87 & 70.22 & 68.60$\pm$2.211 &  69.70\\
 \hline
 Logish & 65.99 & 67.90 & 66.10$\pm$3.621 &  65.79\\ [1ex]
 \hline
\end{tabular}
\end{table}

\begin{table}[]
    \centering
    \caption{CIFAR-10 SqueezeNet AdamW}
    \label{tab:CIFAR_sup_snet_adamw}
    \begin{tabular}{||c c c c c||} 
 \hline
 Name & Train & Valid & Test & Conc\\
 \hline\hline
 TeLU & \textbf{97.42} & \textbf{91.08} & \textbf{90.08}$\pm$0.774 & 88.87\\ 
 \hline
 ReLU & 96.71 & 90.23 & 89.01$\pm$\textbf{0.454} & 89.73\\ 
 \hline
 GELU & 97.25 & 90.66 & 89.63$\pm$0.696 & 88.32\\
 \hline
 Mish & 96.55 & 90.01 & 89.02$\pm$0.866 & 85.81\\
 \hline
 SiLU & 93.90 & 87.99 & 86.70$\pm$1.89 & 78.27\\
 \hline
 Smish & 52.71 & 49.67 & 41.71$\pm$16.9 & 12.99\\
 \hline
 Logish & 53.74 & 51.28 & 42.72$\pm$15.7 & 11.13\\ [1ex]
 \hline
\end{tabular}
\end{table}

\begin{table}[]
    \centering
    \caption{CIFAR-10 SqueezeNet RMSprop}
    \label{tab:CIFAR_sup_snet_rms}
    \begin{tabular}{||c c c c c||} 
 \hline
 Name & Train & Valid & Test & Conc\\
 \hline\hline
 TeLU & 95.19 & \textbf{90.53} & \textbf{89.86}$\pm$0.277 & 90.39\\ 
 \hline
 ReLU & \textbf{96.03} & 89.62 & 88.59$\pm$\textbf{0.138} & 89.08\\ 
 \hline
 GELU & 82.87 & 81.88 & 80.68$\pm$1.184 &  81.60\\
 \hline
 Mish & 91.20 & 88.21 & 87.39$\pm$0.170 &  87.88\\
 \hline
 SiLU & 66.83 & 67.13 & 66.00$\pm$1.315 &  70.20\\
 \hline
 Smish & 68.08 & 68.22 & 66.91$\pm$2.347 &  64.14\\
 \hline
 Logish & 49.40 & 49.59 & 43.20$\pm$18.80 & 53.20 \\[1ex] 
 \hline
\end{tabular}
\end{table}

\begin{table}[]
    \centering
    \caption{CIFAR-10 ResNet18 SGD}
    \label{tab:CIFAR_sup_r18_sgd}
    \begin{tabular}{||c c c c c||} 
 \hline
 Name & Train & Valid & Test & Conc\\
 \hline\hline
 TeLU & \textbf{99.99} & \textbf{93.43} & 92.64$\pm$\textbf{0.076} & 93.30\\ 
 \hline
 ReLU & \textbf{99.99} & 93.30 & 92.65$\pm$0.220 & 93.08\\ 
 \hline
 GELU & \textbf{99.99} & 93.35 & 92.65$\pm$0.248 & 93.15\\
 \hline
 Mish & 99.90 & 93.26 & 92.54$\pm$0.239 & 92.94\\
 \hline
 SiLU & 99.98 & 93.35 & 92.65$\pm$0.239 & 93.04\\
 \hline
 Smish & 99.97 & 93.36 & 92.52$\pm$0.206 &  93.10\\
 \hline
 Logish & 99.86 & 93.41 & \textbf{92.70}$\pm$0.341 & 93.18\\ [1ex] 
 \hline
\end{tabular}
\end{table}

\begin{table}[]
    \centering
    \caption{CIFAR-10 ResNet18 Momentum}
    \label{tab:CIFAR_sup_r18_mom}
    \begin{tabular}{||c c c c c||} 
 \hline
 Name & Train & Valid & Test & Conc \\
 \hline\hline
 TeLU & 99.99 & 94.58 & 94.19$\pm$0.089 & 94.42\\ 
 \hline
 ReLU & \textbf{100.0} & \textbf{95.18} & \textbf{94.71}$\pm$0.163 & 95.02\\ 
 \hline
 GELU & 99.97 & 94.91 & 94.45$\pm$\textbf{0.085} & 94.51\\
 \hline
 Mish & 99.99 & 94.72 & 94.28$\pm$0.233 & 94.56\\
 \hline
 SiLU & 99.84 & 94.72 & 94.07$\pm$0.203 & 94.47\\
 \hline
 Smish & 99.98 & 94.62 & 93.80$\pm$0.176 &  94.44\\
 \hline
 Logish & 99.99 & 94.68 & 93.99$\pm$0.157 &  94.44\\ [1ex]
 \hline
\end{tabular}
\end{table}

\begin{table}[]
    \centering
    \caption{CIFAR-10 ResNet18 AdamW}
    \label{tab:CIFAR_sup_r18_adamw}
    \begin{tabular}{||c c c c c||} 
 \hline
 Name & Train & Valid & Test & Conc\\
 \hline\hline
 TeLU & 99.98 & 92.87 & 92.09$\pm$0.182 & 92.66\\ 
 \hline
 ReLU & \textbf{100.0} & 92.83 & \textbf{92.18}$\pm$\textbf{0.076} & 92.67\\ 
 \hline
 GELU & 99.99 & 92.93 & 92.15$\pm$0.128 & 92.77\\
 \hline
 Mish & 96.55 & 90.01 & 89.02$\pm$0.866 & 92.71\\
 \hline
 SiLU & 93.90 & 87.99 & 86.70$\pm$1.89 & 92.87\\
 \hline
 Smish & 99.96 & \textbf{92.96} & 92.14$\pm$0.291 & 92.79 \\
 \hline
 Logish & 99.99 & 92.89 & 92.15$\pm$0.133 & 92.74 \\ [1ex]
 \hline
\end{tabular}
\end{table}

\begin{table}[]
    \centering
    \caption{CIFAR-10 ResNet18 RMSprop}
    \label{tab:CIFAR_sup_r18_rms}
    \begin{tabular}{||c c c c c||} 
 \hline
 Name & Train & Valid & Test & Conc\\
 \hline\hline
 TeLU & 99.78 & 92.73 & 92.09$\pm$0.251 & 92.51\\ 
 \hline
 ReLU & \textbf{99.85} & 93.03 & 92.40$\pm$\textbf{0.170} & 92.85\\ 
 \hline
 GELU & 99.48 & 93.07 & 92.35$\pm$0.353 & 92.79\\
 \hline
 Mish & 97.92 & 93.67 & \textbf{92.76}$\pm$0.251 & 92.42\\
 \hline
 SiLU & 98.75 & \textbf{93.94} & 92.01$\pm$0.248 & 92.61\\
 \hline
 Smish & 97.31 & 90.71 & 90.06$\pm$0.569 & 90.30\\
 \hline
 Logish & 98.45 & 91.84 & 90.88$\pm$0.299 & 91.45\\ [1ex] 
 \hline
\end{tabular}
\end{table}


\begin{table}[]
    \centering
    \caption{CIFAR-10 ResNet34 SGD}
    \label{tab:CIFAR_sup_r34_sgd}
    \begin{tabular}{||c c c c c||} 
 \hline
 Name & Train & Valid & Test & Conc\\
 \hline\hline
 TeLU & \textbf{99.97} & 91.06 & 90.18$\pm$0.368 & 90.73\\ 
 \hline
 ReLU & 99.96 & 90.54 & 89.55$\pm$0.091 & 90.06\\ 
 \hline
 GELU & \textbf{99.97} & 90.45 & 89.70$\pm$0.325 & 90.04\\
 \hline
 Mish & \textbf{99.97} & 90.84 & 90.18$\pm$\textbf{0.063} & 90.59\\
 \hline
 SiLU & \textbf{99.97} & 90.74 & 89.89$\pm$0.202 & 90.41\\
 \hline
 Smish & \textbf{99.97} & \textbf{91.44} & \textbf{90.90}$\pm$0.185 & 91.06\\
 \hline
 Logish & 99.94 & 91.39 & 90.50$\pm$0.233 & 90.98\\ [1ex] 
 \hline
\end{tabular}
\end{table}

\begin{table}[]
    \centering
    \caption{CIFAR-10 ResNet34 Momentum}
    \label{tab:CIFAR_sup_r34_mom}
    \begin{tabular}{||c c c c c||} 
 \hline
 Name & Train & Valid & Test & Conc\\
 \hline\hline
 TeLU & 99.94 & 94.18 & 93.43$\pm$0.306 & 93.68\\ 
 \hline
 ReLU & \textbf{99.97} & \textbf{94.78} & \textbf{94.07}$\pm$0.216 & 93.86\\ 
 \hline
 GELU & \textbf{99.97} & 94.556 & 93.90$\pm$0.235 & 93.45\\
 \hline
 Mish & 99.93 & 94.16 & 93.46$\pm$0.377 & 93.52\\
 \hline
 SiLU & 99.96 & 94.22 & 93.52$\pm$\textbf{0.148} & 93.21\\
 \hline
 Smish & 99.45 & 93.56 & 92.78$\pm$0.172 &  91.95\\
 \hline
 Logish & 99.64 & 93.82 & 93.07$\pm$0.291 & 92.66 \\ [1ex]
 \hline
\end{tabular}
\end{table}

\begin{table}[]
    \centering
    \caption{CIFAR-10 ResNet34 AdamW}
    \label{tab:CIFAR_sup_r34_adamw}
    \begin{tabular}{||c c c c c||} 
 \hline
 Name & Train & Valid & Test & Conc\\
 \hline\hline
 TeLU & 99.98 & 94.46 & \textbf{93.70}$\pm$\textbf{0.097} & 93.98\\ 
 \hline
 ReLU & 99.98 & 94.33 & 93.53$\pm$0.265 & 93.94\\ 
 \hline
 GELU & 99.98 & 94.21 & 93.59$\pm$0.146 & 93.96\\ 
 \hline
 Mish & \textbf{99.99} & 94.23 & 93.69$\pm$0.201 & 93.80\\
 \hline
 SiLU & \textbf{99.99} & \textbf{94.99} & 93.69$\pm$0.206 & 94.03\\
 \hline
 Smish & 99.95 & 93.55 & 92.61$\pm$0.377 & 93.12 \\
 \hline
 Logish & 99.95 & 93.83 & 92.85$\pm$0.249 & 93.24 \\ [1ex]
 \hline
\end{tabular}
\end{table}

\begin{table}[]
    \centering
    \caption{CIFAR-10 ResNet34 RMSprop}
    \label{tab:CIFAR_sup_r34_rms}
    \begin{tabular}{||c c c c c||} 
 \hline
 Name & Train & Valid & Test & Conc\\
 \hline\hline
 TeLU & 99.68 & 93.49 & 92.51$\pm$0.222 & 93.06\\ 
 \hline
 ReLU & 99.75 & 93.42 & 92.45$\pm$0.170 & 92.94\\ 
 \hline
 GELU & \textbf{99.81} & 93.36 & \textbf{92.97}$\pm$0.196 & 93.19\\ 
 \hline
 Mish & 99.76 & 93.55 & 92.91$\pm$0.194 & 93.00\\
 \hline
 SiLU & 99.68 & \textbf{93.68} & 92.83$\pm$0.264 & 93.20\\
 \hline
 Smish & 99.48 & 92.51 & 91.80$\pm$0.157 & 91.97 \\
 \hline
 Logish & 99.65 & 92.88 & 92.14$\pm$\textbf{0.143} & 92.36 \\ [1ex] 
 \hline
\end{tabular}
\end{table}


\begin{table}[]
    \centering
    \caption{CIFAR-10 ResNet50 SGD}
    \label{tab:CIFAR_sup_r50_sgd}
    \begin{tabular}{||c c c c c||} 
 \hline
 Name & Train & Valid & Test & Conc\\
 \hline\hline
 TeLU & 99.95 & 91.05 & 90.27$\pm$0.160 & 90.66\\ 
 \hline
 ReLU & 99.97 & 90.52 & 89.48$\pm$0.470 & 90.17\\ 
 \hline
 GELU & 99.97 & 90.56 & 89.71$\pm$0.198 & 90.28\\ 
 \hline
 Mish & 99.97 & 91.07 & 90.08$\pm$\textbf{0.161} & 90.66\\
 \hline
 SiLU & 99.97 & 90.84 & 90.02$\pm$0.119 & 90.58\\
 \hline
 Smish & \textbf{99.98} & \textbf{91.62} & \textbf{91.23}$\pm$0.162 & 91.23 \\
 \hline
 Logish & \textbf{99.98} & 91.32 & 90.61$\pm$0.365 & 90.94 \\ [1ex] 
 \hline
\end{tabular}
\end{table}

\begin{table}[]
    \centering
    \caption{CIFAR-10 ResNet50 Momentum}
    \label{tab:CIFAR_sup_r50_mom}
    \begin{tabular}{||c c c c c||} 
 \hline
 Name & Train & Valid & Test & Conc\\
 \hline\hline
 TeLU & \textbf{99.99} & 94.88 & 94.51$\pm$0.225 & 93.35\\ 
 \hline
 ReLU & 99.93 & 94.97 & 94.57$\pm$0.133 & 93.98\\ 
 \hline
 GELU & 99.94 & \textbf{95.04} & \textbf{94.62}$\pm$0.172 & 93.30\\
 \hline
 Mish & 99.95 & 94.80 & 94.45$\pm$0.139 & 93.71\\
 \hline
 SiLU & 99.96 & 94.77 & 94.41$\pm$\textbf{0.102} & 93.02\\
 \hline
 Smish & 98.16 & 93.49 & 92.85$\pm$0.363 & 92.47\\
 \hline
 Logish & 98.98 & 93.82 & 93.36$\pm$0.250 & 92.52\\ [1ex]
 \hline
\end{tabular}
\end{table}

\begin{table}[]
    \centering
    \caption{CIFAR-10 ResNet50 AdamW}
    \label{tab:CIFAR_sup_r50_adamw}
    \begin{tabular}{||c c c c c||} 
 \hline
 Name & Train & Valid & Test & Conc\\
 \hline\hline
 TeLU & 99.89 & 90.83 & 89.83$\pm$0.193 & 90.50\\ 
 \hline
 ReLU & 99.93 & 88.88 & 88.02$\pm$0.309 & 88.51\\ 
 \hline
 GELU & 99.96 & 89.64 & 88.79$\pm$0.374 & 89.19\\ 
 \hline
 Mish & 99.94 & 90.53 & 89.59$\pm$0.343 & 90.28\\
 \hline
 SiLU & 99.97 & 90.61 & 89.73$\pm$0.266 & 90.30\\
 \hline
 Smish & 99.97 & \textbf{91.81} & \textbf{90.96}$\pm$0.189 & 91.50 \\
 \hline
 Logish & \textbf{99.98} & 91.39 & 90.67$\pm$\textbf{0.129} & 91.10 \\ [1ex]
 \hline
\end{tabular}
\end{table}

\begin{table}[]
    \centering
    \caption{CIFAR-10 ResNet50 RMSprop}
    \label{tab:CIFAR_sup_r50_rms}
    \begin{tabular}{||c c c c c||} 
 \hline
 Name & Train & Valid & Test & Conc\\
 \hline\hline
 TeLU & 99.64 & \textbf{93.99} & 93.16$\pm$0.221 & 93.53\\ 
 \hline
 ReLU & 99.73 & 93.52 & 92.84$\pm$0.215 & 93.05\\ 
 \hline
 GELU & 99.66 & 93.78 & 93.06$\pm$0.125 & 93.28\\ 
 \hline
 Mish & \textbf{99.74} & 93.94 & \textbf{93.43}$\pm$\textbf{0.100} & 93.59\\
 \hline
 SiLU & 99.72 & 93.96 & 93.17$\pm$0.159 & 93.54\\
 \hline
 Smish & 99.26 & 91.37 & 90.74$\pm$1.108 & 90.69 \\
 \hline
 Logish & 99.60 & 93.39 & 92.61$\pm$0.249 & 93.06 \\ [1ex] 
 \hline
\end{tabular}
\end{table}


\begin{table}[]
    \centering
    \caption{CIFAR-100 SqueezeNet Hyperparameters}
    \label{tab:CIFAR-100_sup_sq_hps}
    \begin{tabular}{||c c c c||} 
 \hline
 Optimizer & learning rate & weight decay & gamma\\
 \hline\hline
 SGD & 0.04 & 0.003 & 0.2\\ 
 \hline
 Momentum & 0.003 & 0.003 & 0.2\\ 
 \hline
 AdamW & 0.005 & 0.005 & 0.4\\ 
 \hline
 RMSprop & 0.0002 & 0.005 & 0.4\\  [1ex] 
 \hline
\end{tabular}
\end{table}
\begin{table}[]
    \centering
    \caption{CIFAR-100 ResNet18 Hyperparameters}
    \label{tab:CIFAR-100_sup_r18_hps}
    \begin{tabular}{||c c c c||} 
 \hline
 Optimizer & learning rate & weight decay & gamma\\
 \hline\hline
 SGD & 0.05 & 0.003 & 0.2\\ 
 \hline
 Momentum & 0.02 & 0.0008 & 0.4\\ 
 \hline
 AdamW & 0.001 & 0.005 & 0.5\\ 
 \hline
 RMSprop & 0.0001 & 0.0001 & 0.5\\  [1ex] 
 \hline
\end{tabular}
\end{table}
\begin{table}[]
    \centering
    \caption{CIFAR-100 ResNet34 Hyperparameters}
    \label{tab:CIFAR-100_sup_r34_hps}
    \begin{tabular}{||c c c c||} 
 \hline
 Optimizer & learning rate & weight decay & gamma\\
 \hline\hline
 SGD & 0.05 & 0.003 & 0.2\\ 
 \hline
 Momentum & 0.02 & 0.0008 & 0.4\\ 
 \hline
 AdamW & 0.001 & 0.005 & 0.5\\ 
 \hline
 RMSprop & 0.0001 & 0.0001 & 0.5\\  [1ex] 
 \hline
\end{tabular}
\end{table}
\begin{table}[]
    \centering
    \caption{CIFAR-100 ResNet50 Hyperparameters}
    \label{tab:CIFAR-100_sup_r50_hps}
    \begin{tabular}{||c c c c||} 
 \hline
 Optimizer & learning rate & weight decay & gamma\\
 \hline\hline
 SGD & 0.05 & 0.003 & 0.3\\ 
 \hline
 Momentum & 0.01 & 0.0008 & 0.3\\ 
 \hline
 AdamW & 0.0005 & 0.0005 & 0.5\\ 
 \hline
 RMSprop & 0.0001 & 0.001 & 0.5\\  [1ex] 
 \hline
\end{tabular}
\end{table}
\begin{table}[]
    \centering
    \caption{CIFAR-100 SqueezeNet SGD}
    \label{tab:CIFAR-100_sup_sq_sgd}
    \begin{tabular}{||c c c c c||} 
 \hline
 Name & Train & Valid & Test & Conc\\
 \hline\hline
 TeLU & 91.23 & 71.94 & \textbf{71.47}$\pm$\textbf{0.082} & \textbf{70.96}\\ 
 \hline
 ReLU & 96.03 & 69.80 & 69.52$\pm$0.433 & 69.13\\ 
 \hline
 GELU & 88.62 & 67.56 & 67.09$\pm$0.357 & 66.92\\ [1ex] 
 \hline
\end{tabular}
\end{table}

\begin{table}[]
    \centering
    \caption{CIFAR-100 SqueezeNet Momentum}
    \label{tab:CIFAR-100_sup_sq_mom}
    \begin{tabular}{||c c c c c||} 
 \hline
 Name & Train & Valid & Test & Conc\\
 \hline\hline
 TeLU & 92.26 & 70.78 & \textbf{70.53}$\pm$\textbf{0.245} & \textbf{69.72} \\ 
 \hline
 ReLU & 93.90 & 65.36 & 65.05$\pm$0.505 & 64.63 \\ 
 \hline
 GELU & 85.88 & 66.45 & 66.26$\pm$0.288 & 65.57 \\ [1ex]
 \hline
\end{tabular}
\end{table}

\begin{table}[]
    \centering
    \caption{CIFAR-100 SqueezeNet AdamW}
    \label{tab:CIFAR-100_sup_sq_adamw}
    \begin{tabular}{||c c c c c||} 
 \hline
 Name & Train & Valid & Test & Conc\\
 \hline\hline
 TeLU & 99.94 & 70.29 & \textbf{69.64}$\pm$\textbf{0.072} & \textbf{69.56} \\ 
 \hline
 ReLU & 99.90 & 66.81 & 66.31$\pm$0.480 & 66.29 \\ 
 \hline
 GELU & 99.94 & 67.16 & 66.50$\pm$0.444 & 66.58 \\ [1ex]
 \hline
\end{tabular}
\end{table}

\begin{table}[]
    \centering
    \caption{CIFAR-100 SqueezeNet RMSprop}
    \label{tab:CIFAR-100_sup_sq_rms}
    \begin{tabular}{||c c c c c||} 
 \hline
 Name & Train & Valid & Test & Conc\\
 \hline\hline
 TeLU & 89.27 & 69.23 & \textbf{68.83}$\pm$0.331 & 68.46 \\ 
 \hline
 ReLU & 97.64 & 68.40 & 67.99$\pm$\textbf{0.207} & 67.70 \\ 
 \hline
 GELU & 81.71 & 65.57 & 65.19$\pm$0.248 & 65.01 \\ [1ex] 
 \hline
\end{tabular}
\end{table}


\begin{table}[]
    \centering
    \caption{CIFAR-100 ResNet18 SGD}
    \label{tab:CIFAR-100_sup_r18_sgd}
    \begin{tabular}{||c c c c c||} 
 \hline
 Name & Train & Valid & Test & Conc\\
 \hline\hline
 TeLU & 99.62 & 72.93 & 72.87$\pm$0.231 & 72.59 \\ 
 \hline
 ReLU & 99.94 & 74.96 & \textbf{74.70}$\pm$\textbf{0.192} & \textbf{74.70} \\ 
 \hline
 GELU & 99.90 & 74.33 & 74.22$\pm$0.407 & 73.97 \\ [1ex] 
 \hline
\end{tabular}
\end{table}

\begin{table}[]
    \centering
    \caption{CIFAR-100 ResNet18 Momentum}
    \label{tab:CIFAR-100_sup_r18_mom}
    \begin{tabular}{||c c c c c||} 
 \hline
 Name & Train & Valid & Test & Conc\\
 \hline\hline
 TeLU & 99.96 & 75.16 & 75.09$\pm$0.307 & 74.09 \\ 
 \hline
 ReLU & 99.98 & 76.28 & \textbf{76.48}$\pm$\textbf{0.294} & \textbf{75.97} \\ 
 \hline
 GELU & 99.96 & 75.66 & 75.41$\pm$0.384 & 74.95 \\ [1ex]
 \hline
\end{tabular}
\end{table}

\begin{table}[]
    \centering
    \caption{CIFAR-100 ResNet18 AdamW}
    \label{tab:CIFAR-100_sup_r18_adamw}
    \begin{tabular}{||c c c c c||} 
 \hline
 Name & Train & Valid & Test & Conc\\
 \hline\hline
 TeLU & 99.97 & 71.76 & \textbf{71.47}$\pm$\textbf{0.265} & \textbf{71.00} \\ 
 \hline
 ReLU & 99.97 & 71.30 & 71.30$\pm$0.350 & 70.54 \\ 
 \hline
 GELU & 99.97 & 71.19 & 70.99$\pm$0.425 & 70.60 \\ [1ex]
 \hline
\end{tabular}
\end{table}

\begin{table}[]
    \centering
    \caption{CIFAR-100 ResNet18 RMSprop}
    \label{tab:CIFAR-100_sup_r18_rms}
    \begin{tabular}{||c c c c c||} 
 \hline
 Name & Train & Valid & Test & Conc\\
 \hline\hline
 TeLU & 99.85 & 71.36 & 71.23$\pm$0.386 & 70.83 \\ 
 \hline
 ReLU & 99.87 & 71.12 & 70.95$\pm$\textbf{0.078} & 70.37 \\ 
 \hline
 GELU & 99.86 & 71.45 & \textbf{71.32}$\pm$0.324 & 70.99 \\ [1ex]
 \hline
\end{tabular}
\end{table}

\begin{table}[]
    \centering
    \caption{CIFAR-100 ResNet34 SGD}
    \label{tab:CIFAR-100_sup_r34_sgd}
    \begin{tabular}{||c c c c c||} 
 \hline
 Name & Train & Valid & Test & Conc\\
 \hline\hline
 TeLU & 99.85 & 73.51 & 72.95$\pm$0.164 & 72.97 \\ 
 \hline
 ReLU & 99.96 & 75.40 & \textbf{75.23}$\pm$\textbf{0.108} & \textbf{75.13} \\ 
 \hline
 GELU & 99.91 & 74.52 & 74.14$\pm$0.366 & 74.15 \\ [1ex] 
 \hline
\end{tabular}
\end{table}

\begin{table}[]
    \centering
    \caption{CIFAR-100 ResNet34 Momentum}
    \label{tab:CIFAR-100_sup_r34_mom}
    \begin{tabular}{||c c c c c||} 
 \hline
 Name & Train & Valid & Test & Conc\\
 \hline\hline
 TeLU & 99.96 & 74.93 & 74.94$\pm$0.305 & 74.40 \\ 
 \hline
 ReLU & 99.98 & 77.30 & \textbf{76.93}$\pm$\textbf{0.178} & \textbf{76.99} \\ 
 \hline
 GELU & 99.96 & 75.77 & 75.38$\pm$0.322 & 75.16 \\ [1ex]
 \hline
\end{tabular}
\end{table}

\begin{table}[]
    \centering
    \caption{CIFAR-100 ResNet34 AdamW}
    \label{tab:CIFAR-100_sup_r34_adamw}
    \begin{tabular}{||c c c c c||} 
 \hline
 Name & Train & Valid & Test & Conc\\
 \hline\hline
 TeLU & 99.97 & 71.88 & \textbf{71.73}$\pm$0.350 & \textbf{71.18} \\ 
 \hline
 ReLU & 99.97 & 71.60 & 71.60$\pm$\textbf{0.284} & 71.06 \\ 
 \hline
 GELU & 99.97 & 71.49 & 71.29$\pm$0.375 & 71.00 \\ [1ex]
 \hline
\end{tabular}
\end{table}

\begin{table}[]
    \centering
    \caption{CIFAR-100 ResNet34 RMSprop}
    \label{tab:CIFAR-100_sup_r34_rms}
    \begin{tabular}{||c c c c c||} 
 \hline
 Name & Train & Valid & Test & Conc\\
 \hline\hline
 TeLU & 99.78 & 72.10 & 72.01$\pm$\textbf{0.080} & 71.65 \\ 
 \hline
 ReLU & 99.74 & 71.96 & 71.91$\pm$0.262 & 70.89 \\ 
 \hline
 GELU & 99.78 & 72.31 & \textbf{72.10}$\pm$0.247 & 71.37 \\ [1ex]
 \hline
\end{tabular}
\end{table}


\begin{table}[]
    \centering
    \caption{CIFAR-100 ResNet50 SGD}
    \label{tab:CIFAR-100_sup_r50_sgd}
    \begin{tabular}{||c c c c c||} 
 \hline
 Name & Train & Valid & Test & Conc\\
 \hline\hline
 TeLU & 99.92 & 76.99 & 76.77$\pm$0.258 & 76.43 \\ 
 \hline
 ReLU & 99.95 & 77.18 & \textbf{77.14}$\pm$\textbf{0.110} & 76.60 \\ 
 \hline
 GELU & 99.91 & 77.22 & 76.56$\pm$0.127 & 76.76\\ [1ex] 
 \hline
\end{tabular}
\end{table}

\begin{table}[]
    \centering
    \caption{CIFAR-100 ResNet50 Momentum}
    \label{tab:CIFAR-100_sup_r50_mom}
    \begin{tabular}{||c c c c c c c||} 
 \hline
 Name & Train & Valid & Test & Conc\\
 \hline\hline
 TeLU & 99.98 & 76.48 & \textbf{76.57}$\pm$\textbf{0.200} & \textbf{76.25} \\ 
 \hline
 ReLU & 99.98 & 75.12 & 75.08$\pm$0.270 & 74.68 \\ 
 \hline
 GELU & 99.97 & 75.76 & 75.67$\pm$0.309 & 75.20 \\ [1ex]
 \hline
\end{tabular}
\end{table}

\begin{table}[]
    \centering
    \caption{CIFAR-100 ResNet50 AdamW}
    \label{tab:CIFAR-100_sup_r50_adamw}
    \begin{tabular}{||c c c c c||} 
 \hline
 Name & Train & Valid & Test & Conc\\
 \hline\hline
 TeLU & 99.97 & 75.06 & \textbf{74.75}$\pm$0.266 & \textbf{74.40} \\ 
 \hline
 ReLU & 99.95 & 73.79 & 73.52$\pm$\textbf{0.200} & 73.11 \\ 
 \hline
 GELU & 99.97 & 74.17 & 73.81$\pm$0.340 & 73.62 \\ [1ex]
 \hline
\end{tabular}
\end{table}

\begin{table}[]
    \centering
        \caption{CIFAR-100 ResNet50 RMSprop}
    \label{tab:CIFAR-100_sup_r50_rms}
    \begin{tabular}{||c c c c c||} 
 \hline
 Name & Train & Valid & Test & Conc\\
 \hline\hline
 TeLU & 99.79 & 74.02 & \textbf{74.02}$\pm$0.142 & \textbf{73.23} \\ 
 \hline
 ReLU & 99.73 & 72.88 & 72.26$\pm$0.478 & 71.70 \\ 
 \hline
 GELU & 99.73 & 73.42 & 72.75$\pm$\textbf{0.131} & 72.32 \\  [1ex]
 \hline
\end{tabular}
\end{table}



\begin{table}[]
    \centering
        \caption{TinyImageNet200 ResNet34 Hyperparameters}
    \label{tab:tiny_sup_r34_hps}
    \begin{tabular}{||c c c c||} 
 \hline
 Optimizer & learning rate & weight decay & gamma\\
 \hline\hline
 SGD & 0.05 & 0.001 & 0.3\\ 
 \hline
 Momentum & 0.04 & 0.0004 & 0.4\\ 
 \hline
 AdamW & 0.0005 & 0.004 & 0.5\\ 
 \hline
 RMSprop & 0.0001 & 0.0002 & 0.6\\  [1ex]
 \hline
\end{tabular}
\end{table}

\begin{table}[]
    \centering
        \caption{TinyImageNet200 ResNet34 SGD}
    \label{tab:tiny_sup_r34_sgd}
    \begin{tabular}{||c c c c||} 
 \hline
 Name & Top-1 Test & Top-5 Test & Conc\\
 \hline\hline
 TeLU & \textbf{62.34}$\pm$\textbf{0.173} & \textbf{81.86}$\pm$0.337 & \textbf{61.99}\\ 
 \hline
 ReLU & 61.16$\pm$0.314 & 80.51$\pm$\textbf{0.263} & 60.88\\ [1ex]
 \hline
\end{tabular}
\end{table}

\begin{table}[]
    \centering
        \caption{TinyImageNet200 ResNet34 Momentum}
    \label{tab:tiny_sup_r34_mom}
    \begin{tabular}{||c c c c||} 
 \hline
 Name & Top-1 Test & Top-5 Test & Conc\\
 \hline\hline
 TeLU & \textbf{62.09}$\pm$\textbf{0.222} & \textbf{82.28}$\pm$\textbf{0.453} & \textbf{61.93}\\ 
 \hline
 ReLU & 38.37$\pm$34.6 & 50.32$\pm$43.8 & 38.28\\ [1ex]
 \hline
\end{tabular}
\end{table}

\begin{table}[]
    \centering
        \caption{TinyImageNet200 ResNet34 AdamW}
    \label{tab:tiny_sup_r34_adamw}
    \begin{tabular}{||c c c c ||} 
 \hline
 Name & Top-1 Test & Top-5 Test & Conc\\
 \hline\hline
 TeLU & 54.04$\pm$0.822 & \textbf{76.04}$\pm$0.626 & 53.62 \\ 
 \hline
 ReLU & \textbf{54.88}$\pm$\textbf{0.720} & 75.70$\pm$\textbf{0.592} & \textbf{54.40}\\ [1ex]
 \hline
\end{tabular}
\end{table}

\begin{table}[]
    \centering
    \caption{TinyImageNet200 ResNet34 RMSprop}
    \label{tab:tiny_sup_r34_rms}
    \begin{tabular}{||c c c c||} 
 \hline
 Name & Top-1 Test & Top-5 Test & Conc \\
 \hline\hline
 TeLU & \textbf{58.48}$\pm$\textbf{0.034} & \textbf{78.83}$\pm$0.380 & \textbf{57.93}\\ 
 \hline
 ReLU & 58.33$\pm$0.271 & 78.46$\pm$\textbf{0.263} & 57.18\\ [1ex]
 \hline
\end{tabular}
\end{table}

\end{document}